%% file: main.tex
\documentclass[journal]{IEEEtran}

\usepackage{subcaption}
\usepackage{epsfig} %% for loading postscript figures
\usepackage{bm}
\usepackage{amssymb}
\usepackage{amsmath}
\DeclareMathOperator*{\argmax}{arg\,max}

\usepackage{cite}
\usepackage{amsthm}
\usepackage{commath}
\usepackage{color}
\usepackage{xcolor}

\usepackage{flushend}
\newcommand{\comment}[1]{}
\newtheorem{proposition}{Proposition}
\newtheorem{remark}{Remark} 	
\newtheorem{definition}{Definition} 	
\begin{document}
\title{Rigid Body Dynamic Simulation with Line and Surface Contact}

\author{Jiayin~Xie,~\IEEEmembership{Student Member,~IEEE,}
        Nilanjan~Chakraborty ,~\IEEEmembership{Member,~IEEE,}
\thanks{Jiayin~Xie and Nilanjan~Chakraborty are with the Department of Mechanical Engineering, State University of New York at Stony Brook, Stony Brook,
NY, 11790 USA. Email: jiayin.xie@stonybrook.edu; nilanjan.chakraborty@stonybrook.edu.}}% <-this % stops a space
%%% first author
%\author{Jiayin Xie
% %    \affiliation{
% 	Department of Mechanical Engineering\\
% 	State University of New York at Stony Brook\\
% 	Stony Brook, NY 11790\\
%     Email: jiayin.xie@stonybrook.edu
%     }	
% }

% %%% second author
% %%% remove the following entry for single author papers
% %%% add more entries for additional authors
% \author{Nilanjan Chakraborty 
%     \affiliation{ 
%     Department of Mechanical Engineering\\
% 	State University of New York at Stony Brook\\
% 	Stony Brook, NY 11790\\
%     Email: nilanjan.chakraborty@stonybrook.edu
%     }
% }

%%% third author
%%% remove the following entry for single author papers
%%% add more entries for additional authors

\maketitle  
%%%%%%%%%%%%%%%%%%%%%%%%%%%%%%%%%%%%%%%%%%%%%%%%%%%%%%%%%%%%%%%%%%%%%%
\begin{abstract}
In this paper, we develop a principled method to model line and surface contact with point contact (we call this point, equivalent contact point) that is consistent with physics-based models of surface (line) contact. Assuming that the set of contact points form a convex set, we solve the contact detection and dynamic simulation step simultaneously by formulating the problem as a mixed nonlinear complementarity problem. This allows us to simultaneously compute the equivalent contact point as well as the wrenches (forces and moments) at the equivalent contact point (consistent with the friction model) along with the configuration and velocities of the rigid objects. Furthermore, we prove that the contact constraints of no inter-penetration between the objects is also satisfied. We present a geometrically implicit time-stepping scheme for dynamic simulation for contacts between two bodies with convex contact area, which includes  line contact and surface contact. We prove that for surface and line contact, for any value of the velocity of center of mass of the object, there is a unique solution for contact point and contact wrench that satisfies the discrete-time equations of motion. Simulation examples are shown to demonstrate the validity of our approach and show that with our approach we can seamlessly transition between point, line, and surface contact.
%\end{abstract}

%{\em Note to Practitioners}:
%\begin{abstract}
% This paper is motivated by contact modeling and simulation applications where the objects in intermittent contact can be assumed to be nominally rigid bodies, but the contact between the objects cannot be idealized as point contacts. Traditional methods of dynamic simulation treat the collision detection and dynamic state update problem separately. Therefore, for line and surface contact there are an infinite number of solutions for selecting the contact point and there is no clear criterion to computationally select one (or a set) of contact point(s) to represent the contact. This paper presents a new approach for modeling line and surface contact by formulating and solving a system of equations that combines collision detection with dynamic state update. We prove that such a method allows us to select a contact point and contact wrenches uniquely, while also guaranteeing that there is no non-physical effects in simulation like penetration between rigid bodies. We present simulation results showing that our method can model non-point contact and also allows transition between point, line, and surface contact automatically.
\end{abstract}

\begin{IEEEkeywords}
Dynamic Simulation, Nonlinear Complementarity Problems,  Contact Modeling, Intermittent Contact.
\end{IEEEkeywords}
\input{1_Introduction.tex}
\input{2_related_work.tex}
\input{3_dynamic_model.tex}
\input{4_contact_constraints.tex}
\input{5_well_posed.tex}
\input{6_result.tex}
\input{7_conclusion.tex}
\bibliographystyle{IEEEtran}

% Here's where you specify the bibliography database file.
% The full file name of the bibliography database for this
% article is asme2e.bib. The name for your database is up
% to you.
%\bibliography{jiayin}
\input{main.bbl}

\end{document}

%% file: 1_Introduction.tex
\section{Introduction}
A fundamental characteristic of  a wide range of robotics problems including robotic grasping, in-hand manipulation~\cite{ChavanR15, MaD11},  non-prehensile manipulation (say, by pushing)~\cite{Lynch1996, Peshkin1988a} is the presence of controlled intermittent contact between the gripper or the robotic end effector with other objects. 
The ability to predict motion of objects undergoing intermittent contact can help in the design of robust grasp strategies, effective part feeder devices~\cite{SongTVP04}, and planning and control algorithms for manipulation. Thus, dynamic modeling and simulation for problems with intermittent unilateral contact is a key problem. 

%{\color{red} In this paper, we assume bodies in contact to be rigid. Rigid body assumption can be viewed as an approximation to the reality, because no objects are perfectly rigid. It is true that when objects are deformable, we need elastic models~\cite{} to model the contacts and the resulting deformation of bodies in contact area.  However, rigid body assumption is reasonable for applications in manufacturing and robotics, where the deformation between objects in contact is smaller enough to ignore. For example, let a manipulator manipulate a T-shaped bar, where the bar has intermittent contact with the ground. Both the bar and the ground are made of solid materials, and the deformation between them is negligible. We want to predict the motion of the bar assuming the bodies in contact are rigid which is a suitable approximation in this case (see~\cite{} for detail). Furthermore, compared to elastic model which requires solving complex system of partial differential equations, rigid body model is more convenient.}

\begin{figure}
\includegraphics[width=1\columnwidth]{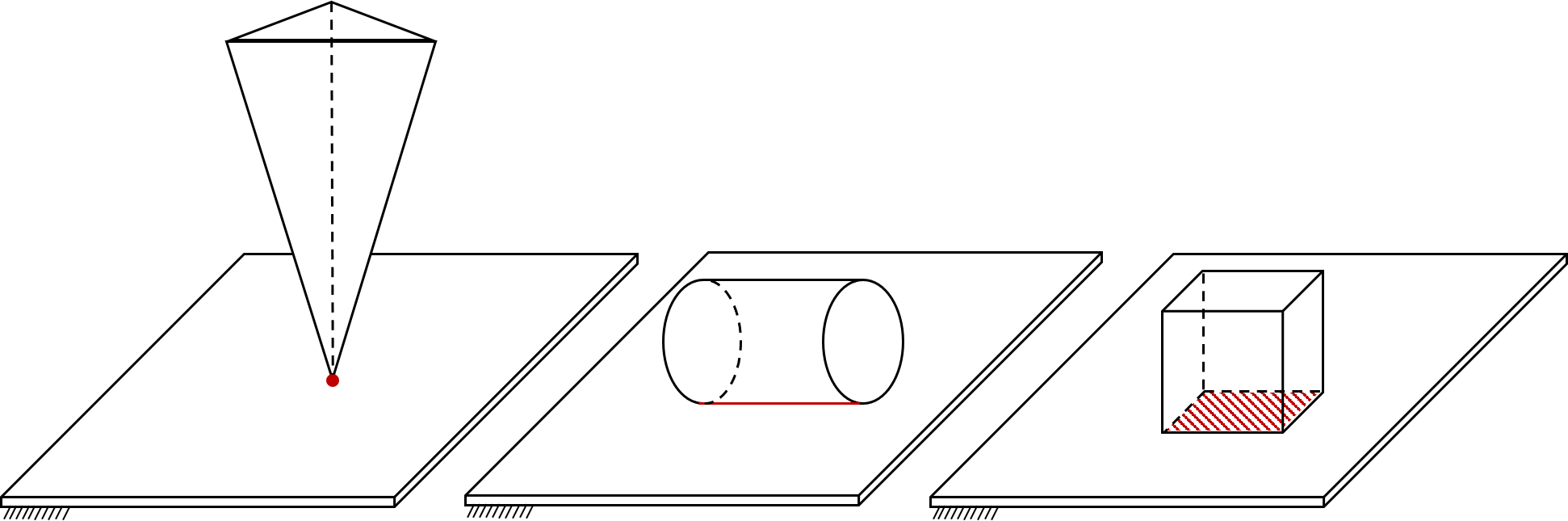}
\caption{State-of-the-art dynamic simulation algorithms assume that the contact between two objects is always a point contact (left). However, the contact  may actually be a (middle) line or (right) surface contact. We develop a principled method to incorporate line and surface contact in rigid body dynamic simulation. }
\label{Figure::Introduction_example}
\end{figure}

Existing mathematical models for motion of
% NC added "rigid" below 
rigid objects with intermittent contact~\cite{PfeifferG1996,LeineN2013,AcaryB2008,PangS2008,Moreau1988,JeanM1987,PaoliS1999} can be classified into two broad categories, namely, (a) Differential Algebraic Equation (DAE) models~\cite{Haug1986} and (b) Differential Complementarity Problem (DCP) models~\cite{Cottle2009,Trinkle1997, PfeifferG08}. In DAE models, it is assumed that the contact mode (i.e., sliding contact, rolling contact, or no contact) is known, whereas DCP models solve for the contact mode along with the state of the system. Irrespective of whether a DAE or DCP is used to model the dynamical system, an almost universal assumption is that the contact between the two objects is a point contact. 
%NC new stuff begins

However, in practical manipulation scenarios the point contact assumption may not be valid.

Consider the example of a cylinder being pushed on a flat surface. The contact between the flat surface and the cylinder is a line contact (Figure~\ref{Figure::Introduction_example}, middle) that changes with time as the cylinder rolls. Further, consider a box being pushed on a table. Here, the contact between the box and the table is a surface contact (Figure~\ref{Figure::Introduction_example}, right).

%NC end new stuff
%Figure~\ref{Figure::Introduction_example} shows schematic sketch of a few types of contacts, namely point contact, line contact, and surface area contact that arise for many practical applications in manipulation. For example, for manipulating a cylindrical or a polyhedral object by pushing on a plane, the point contact assumption is not valid. The contact between the cylindrical or polyhedral object may be line (Figure~\ref{Figure::Introduction_example}, middle) or surface (Figure~\ref{Figure::Introduction_example}, right) contact. 
In general, during motion, the contact between two objects may switch among point, line, and surface contact. 
%NC new stuff begins

In such cases, multiple contact points (usually less than three) are usually chosen in an ad hoc manner. For example, for the box on the table one can choose any three points, such that the projection of the center of mass of the box on the flat surface lies within the convex hull of the points.  If more than three contact points are chosen (e.g., the four vertices), the force distribution at the points cannot be uniquely determined. 
%(e.g., choosing the two ends of a line segment to represent a line contact).
However, such ad hoc {\em a priori} choices may not be valid because the actual contact patch may change during motion, e.g., when part of the box is over the edge of the table. Also, the contact mode may change from surface to point or line, e.g., when the box is being tilted.

%NC new stuff ends
Furthermore, such ad hoc choices (along with other assumptions like linearization of friction cones) can lead to inaccuracies in simulation (please see Figure~\ref{figure:ex2_delta_2} and ~\ref{figure:ex2_delta_3}, which demonstrates this for Open Dynamics Engine (ODE)~\cite{SmithODE} and BULLET~\cite{CouBullet}, two popular dynamic simulation software codes).
% NC change back
%{\color{red} software codes}). 
Hence, simulation for predicting motion for line or surface contact may not be reliable enough for use in manipulation planning or control.  

A key reason for the  modeling difficulties with line or surface contact is that most current dynamic simulation methods decouple the contact detection from integrating the equations of motion. A collision detection algorithm is used to compute the closest points or the contact points  at a given configuration obtained by integrating the equations of motion. When there is line or surface contact, there are  infinitely many possible pairs of closest points on the two objects and thus the collision detection problem does not have a unique solution. In such cases, it is not clear which point should be used as a contact point to ensure that there is no non-physical penetration and error in the simulation is not introduced by the choice of the point. {\em The goal of this paper is to develop a principled method to incorporate line and surface contact in dynamic simulation}. The set of contact points will often be referred to as {\em contact patch} and we assume that the contact patch is a convex set.

When there is a contact patch between two rigid objects, there is a distribution of the normal force and the friction force in the contact patch. There is a unique point in the contact patch where the net moment due to the normal contact force is zero.  The effect of the contact patch can be modeled equivalently by the sum of the total distributed normal and tangential force acting at this single point and the net moment about this point due to the tangential contact forces. We call this point the {\em equivalent contact point} (ECP).
%NC change back
%{\color{red} (ECP)}. 

In statics, where the friction may or may not be relevant, the point in the contact patch where the net moment of the normal force distribution is zero is called the centre of pressure. In the manipulation literature, where the friction is also relevant, this point is called the center of friction~\cite{Mason01}.  In this paper, we show that the ECP as well as the contact force and moment (i.e., contact wrench) at the ECP can be computed by incorporating the collision detection within the dynamic simulation time step. We use a DCP formulation of the dynamics, since it does not make any assumptions about the contact modes, which are usually not known {\em a priori}. In~\cite{NilanjanChakraborty2007}, the authors  presented a method for incorporating the collision detection within the dynamic simulation time step. They showed that such a formulation improves accuracy of the contact dynamics simulation by preventing interpenetration between rigid objects. This method was called the {\em geometrically implicit time-stepping method}. In this paper, we show that the system of equations derived in~\cite{NilanjanChakraborty2007} can give the ECP and equivalent contact wrench for non-point contact. Furthermore, it is guaranteed that there will be no interpenetration between the objects at the end of the time step.  
%{\color{red} The geometrically implicit method is implicit in the geometric information (when the distance function between objects is not available in closed form) by incorporating body geometry in the dynamic time-stepping sub-problem.} However, the developed model was applied for point contact only. 

%NC Stuff added

{\bf Modeling Assumptions}: We make the following modeling assumptions in this paper:
(1) All objects are assumed to be rigid bodies. (2) Geometrically, each object is a convex set modeled as intersection of convex inequalities, and the contact patch is a planar convex set. (3) The net friction wrench (force and moment) at a contact satisfies a generalized Coulomb's friction model. (4) The motion generated is such that the principle of maximum power dissipation holds. We discuss the friction model in detail in Section~\ref{sec:fric}. Note that the geometric assumption of convexity of objects and convexity of contact patch is for convenience of presentation. These assumptions can be relaxed as long as the contact patch is planar (please see~\cite{XieC19}).
%NC end changes
%, which assumes that (a) the contact force and moment lies within a convex cone and (b) the friction force is independent of contact area and only dependent on magnitude of normal force.
%we show that the system of equations derived in~\cite{NilanjanChakraborty2007} can give the ECP and equivalent contact wrench for non-point contact. Furthermore, it is guaranteed that there will be no interpenetration between the objects at the end of the time step.  

{\bf Contributions}: Our key contributions are as follows: (a) We show that the geometrically implicit time-stepping model proposed in~\cite{NilanjanChakraborty2007} can be used to simulate contact problems with line or surface contact. We also formally prove that the geometric contact constraints will be always satisfied and there will be no interpenetration between the objects.
%(b) For $3D$ line and surface contact with pure translation or pure rotation, we derive a closed form solution for the states as well as %the contact point and the contact wrenches. 
(b) We formalize our claim that our approach of using a geometrically implicit time-stepping model by combining collision detection within the equations of motion leads to a well-posed problem. More specifically, for $3D$ line and surface contact, we prove that for any value of the state (position, orientation, linear, and angular velocity) of the two objects, there is a unique solution for the closest points and the contact wrenches. Thus, for resolving non-point contacts, when we solve for the contact wrenches and contact point simultaneously, the problem is well-posed, whereas it is ill-posed when we want to resolve contacts by solving for the contact points separately (as is traditionally done), since there are infinite number of possible contact points for the given state. 
% * <jiayin.xie@stonybrook.edu> 2017-09-22T04:44:38.218Z:
% 
% Equivalent contact point (ECP)
% 
% ^.
(c) We also present numerical simulation results, depicting the correctness of our method by comparing the numerical solution of the dynamic simulation with analytical solutions that we have derived for pure translation with patch contact. For more general motion with non-point contact, we show that our algorithm can track the change of the {\em equivalent contact point} as the contact mode changes from point contact to line contact to surface contact. We use the PATH solver \cite{StevenP.Dirkse1995} to get the numerical solutions to the complementarity problem that we are formulating. 
%{\color{red} Our results are dependent on the rigid body assumption and a generalized Coulomb friction model, which assumes that (a) the contact force and moment lies within a convex cone and (b) the friction force is independent of contact area and only dependent on magnitude of normal force. The results are not dependent on the exact knowledge of contact geometry or pressure distribution on the contact patch.} A preliminary version of this work was presented in~\cite{XieC16}. In this paper, we expand on the paper in~\cite{XieC16}, by including a complete proof of Proposition $1$ and substantially modifying Section~\ref{sec:ecp}.

This paper is organized as follows: In Section~\ref{sec:rw}, we discuss in detail the relationship of our work to the related work. In Section~\ref{sec:dynmod}, we present the mathematical model of the equations of motion of objects in intermittent contact with each other. The modeling of the contact constraints is presented in detail in~\ref{sec:modcc}. In Section~\ref{sec:ecp}, we formally prove that the discrete time-model that we present by combining the equations corresponding to the computation of the closest point within the equations of motion leads to a well-posed problem for determining the ECP and contact wrench. In Section~\ref{sec:res}, we present our numerical simulation results depicting the performance of our algorithm and in Section~\ref{sec:conc}, we present our conclusions and outline avenues of future research.

%% file: 2_related_work.tex
\section{RELATED WORK}
\label{sec:rw}

In this paper, we are concerned with dynamic simulation of rigid bodies that are in intermittent contact with each other and we use a DCP model.  
Let ${\bf u}\in \mathbb{R}^{n_1}$,  ${\bf v}\in \mathbb{R}^{n_2}$ and let ${\bf g}$ :$ \mathbb{R}^{n_1}\times \mathbb{R}^{n_2} \rightarrow \mathbb{R}^{n_1} $, ${\bf f}$ : $ \mathbb{R}^{n_1}\times \mathbb{R}^{n_2} \rightarrow \mathbb{R}^{n_2}$ be two vector functions and the notation $0 \le {\bf x} \perp {\bf y} \ge 0$ imply that ${\bf x}$ is orthogonal to ${\bf y}$ and each component of the vectors is non-negative. 
\begin{definition}
\cite{Facchinei2007} The differential (or dynamic) complementarity problem is to find ${\bf u}$ and ${\bf v}$ satisfying
$$\dot{{\bf u}} = {\bf g}({\bf u},{\bf v}), \ \ \ 0\le {\bf v} \perp{ \bf f}({\bf u},{\bf v}) \ge 0 $$
\end{definition}
\begin{definition}
The mixed complementarity problem is to find ${\bf u}$ and ${\bf v}$ satisfying
$${\bf g}({\bf u},{\bf v})={\bf 0}, \ \ \ 0\le {\bf v} \perp {\bf f}({\bf u},{\bf v}) \ge 0$$
if the functions ${\bf f}$ and ${\bf g}$ are linear the problem is called a mixed linear complementarity problem (MLCP), otherwise, the problem is called a mixed nonlinear complementarity problem (MNCP). As we will discuss later, our continuous time dynamics model is a DCP whereas our discrete-time dynamics model is a MNCP. 
\end{definition}

Modeling the intermittent contact between bodies in motion as a complementarity constraint was first done by Lotstedt~\cite{Lotstedt82}. Subsequently, there was a substantial amount of effort in modeling and dynamic simulation with complementarity constraints~\cite{AnitescuCP96, Pang1996, StewartT96, Liu2005, DrumwrightS12, Todorov14, AcaryB2008, PfeifferG1996, Studer2009, LeineN2013, capobiancoE2018, BrulsAC2018}. The DCP that models the equations of motion, usually, cannot be solved in closed form. Therefore,  time-stepping schemes~\cite{Studer2009,capobiancoE2018,BrulsAC2018, SchindlerR+2015} have been introduced to solve the DCP. The time-stepping problem is: {\em given the state of the system and applied forces, compute
an approximation of the system one time step into the future.} Repeatedly solving this problem
results in an approximation of the solution of the equations of motion of the system. There are different assumptions that are made when forming the discrete equations of motion that makes the system Mixed Linear Complementarity problem (MLCP)~\cite{AnitescuP97, AnitescuP02} or mixed non-linear complementarity problem (MNCP)~\cite{Tzitzouris01,NilanjanChakraborty2007}. The MLCP problem can be solved faster but it assumes that the friction cone constraints are linearized and that the distance function between two bodies (which is a nonlinear function of the configuration) is also linearized. 

Depending on whether the distance function is approximated, the time-stepping schemes can also be divided into geometrically explicit schemes~\cite{AnitescuCP96, StewartT96} and geometrically implicit schemes~\cite{Tzitzouris01, Laursen2013,Wriggers2004,AcaryB2008}. In geometrically explicit schemes, at the current state, a collision
detection routine is called to determine separation or penetration distances between the bodies,
but this information is not incorporated as a function of the unknown future state at the end of
the current time step. A goal of a typical time-stepping scheme is to guarantee consistency of the
dynamic equations and all model constraints at the end of each time step. However, since the
geometric information is obtained and approximated only at the start of the current time-step,
then the solution will be in error.  Thus we use a geometrically implicit time stepping scheme in this paper and the resulting discrete time problem is a MNCP.

A common aspect of all the time stepping schemes mentioned above is that they assume the contact between the two objects to be a point contact. In many application scenarios in grasping and manipulation, line and surface contact with friction is ubiquitous. Therefore much attempt has been made to model and understand the effect of non-point frictional contact in robotic manipulation~\cite{Erdmann1994, Goyal1991, Howe1996}. Building on these friction models, there is work in robotic manipulation based on quasi-static assumption where the inertial forces can be neglected with respect to the friction forces~\cite{KaoC92, Lynch1996, Peshkin1988, Peshkin1988a}. In this paper we use a friction model for general dynamic simulation, where the quasi-static assumption may not be valid. We want a general model that considers moments due to the distributed friction force; hence we use  a generalization of Coulomb's friction.
%and is known as the soft-finger contact model in robotics~\cite{MLS94}.

%NC changed
%{\color{red} There has been a long history for modeling frictional collisions between objects~\cite{Keller1986,Routh1955}. One of the approach is called coefficient of restitution based method, in which the process of collision is modeled by the ratio of the velocity before contact to that after contact. The coefficient of restitution normally ranges from 0 to 1, where 1 represents perfectly elastic collision and 0 would be perfectly inelastic collision. Another approach is called force based method, where the contact compliance is modeled as a spring-damper system. A simple example would be a linear spring-damper system ($F = k\delta + c\dot{\delta}$) , i.e., the contact force $(F)$ depends on the contact deformation ($\delta$) and velocity ($\dot{\delta}$). In this paper, we choose the former method to model the contact between two rigid bodies. We apply a perfect inelastic or plastic law and coefficient of restitution chosen is zero. }

%% file: 3_dynamic_model.tex
\section{Dynamic Model}
\label{sec:dynmod}
We use a geometrically implicit optimization-based time-stepping scheme to model the dynamic simulation of intermittent unilateral contact between two rigid objects where the equations in each time step form a mixed complementarity problem. It is made up of the following parts: (a) Newton-Euler equations (b) kinematic map relating the generalized velocities to the linear and angular velocities (c) friction law and (d) contact constraints incorporating the geometry of the objects. The parts (a), (b), and (c) are standard for any complementarity-based formulation (although part (c) can change depending on the assumptions (or approximations) made about the friction model). Part (d) is the key aspect of the model in the context of this paper and was introduced in~\cite{NilanjanChakraborty2007} for point contacts.

For simplicity of exposition, we will introduce the notations and describe the equations of motion for a single object in a single (convex) patch contact with another object. For multiple moving objects, we can always form the equations of motion of the overall system from the description below along with the fact that the contact wrenches at the contact point are equal and opposite to each other for any two contacting objects. 
%This is without loss of generality as we can always form the equations of motion of the overall system from the description below along with the fact that the contact forces and moments at the contact point are equal and opposite to each other for any two contacting objects. 
Let $\bf{q}$ be the position of the center of mass of the object and the orientation of the object in an inertial frame ($\bf{q}$ can be $6 \times 1$ or $7\times 1$ vector depending on the representation of the orientation). We will use unit quaternion to represent the orientation unless otherwise stated.
Let $\bm{\nu}$ be the generalized velocity concatenating the linear ($\bm{v}$) and spatial angular ($^s\bm{\omega}$) velocities. Let
$\lambda_n$ ($p_n$) be the magnitude of normal contact force (impulse) at the equivalent contact point (ECP), $\lambda_t$ ($p_t$) and $\lambda_o$ ($p_o$) be the orthogonal components of the friction force (impulse) on the tangential plane at ECP, and $\lambda_r$ ($p_r$) be the frictional force (impulse) moment about the contact normal.
\subsection{Newton-Euler Equations}
The  Newton-Euler equations in an inertial frame $\mathcal{F}_w$ are as follows:
\begin{equation} \label{eq1}
{\bf M}({\bf q})
{\dot{\bm{\nu}}} = 
{\bf W}_{n}\lambda_n+
{\bf W}_{t}\lambda_{t}+
{\bf W}_{o} \lambda_{o}+
{\bf W}_{r}\lambda_{r}+
{\bf F}_{app}+{\bf F}_{vp}
\end{equation}
\noindent
%NC change back
%{\color{red} where $\bm{M}(\bm{q}) $ is the generalized inertia matrix.} 
where ${\bf M}({\bf q}) = \mathrm{diag}(m\mathbf{I}_3, \mathbf{I}') \in \mathbb{R}^{6 \times 6}$ is the generalized inertia matrix, $m$ is the object mass, $\mathbf{I}' \in \mathbb{R}^{3 \times 3}$ is the object inertia matrix about the CM and with respect to the inertial frame $\mathcal{F}_w$, $\mathbf{I}_3$ is a $3 \times 3$ identity matrix.
Here ${\bf F}_{app}  \in \mathbb{R}^6$ is the vector of external forces and moments (including gravity and excluding the contact normal force and frictional forces/moments), ${\bf F}_{vp}  \in \mathbb{R}^6$ is the vector of Coriolis and centripetal forces and moments, ${\bf W}_{n}$, ${\bf W}_{t}$, ${\bf W}_{o}$ and $
{\bf W}_{r} \in \mathbb{R}^6$ are the unit wrenches of the normal contact forces, frictional contact forces, and frictional moments.
%NC change back
%{\color{red}  Let the ECP be the origin of the contact frame. Let the normal axis of the contact frame be the unit vector $\bm{n}\in R^3$, which is orthogonal to the contact patch. The tangential axes of the frame are unit vectors $\bm{t} \in R^3 $ and $\bm{o}\in R^3$, which is orthogonal to the normal axis  $\bm{n}$. } 
Let the ECP be the origin of the contact frame. Let the normal axis of the contact frame be the unit vector ${\bf n}\in \mathbb{R}^3$, which is orthogonal to the contact patch. The tangential axes of the frame are unit vectors ${\bf t} \in \mathbb{R}^3 $ and ${\bf o}\in \mathbb{R}^3$, which is orthogonal to the normal axis  ${\bf n}$.
Let ${\bf r}$ be the vector from center of gravity to the ECP, expressed in the inertial frame:
\begin{equation}
\begin{aligned}
\label{equation:wrenches}
{\bf W}_{n} =  \left [ \begin{matrix} 
{\bf n}\\
{\bf r}\times {\bf n}
\end{matrix}\right]
\quad 
{\bf W}_{t} =  \left [ \begin{matrix} 
{\bf t}\\
{\bf r}\times {\bf t}
\end{matrix}\right]
\\
{\bf W}_{o} =  \left [ \begin{matrix} 
{\bf o}\\
{\bf r}\times {\bf o}
\end{matrix}\right]
\quad 
{\bf W}_{r} =  \left [ \begin{matrix} 
{\bf 0}\\
\ \ {\bf n} \ \
\end{matrix}\right]
\end{aligned}
\end{equation}
Here $\bf 0$ is a $3 \times 1$ vector with each entry equal to $0$.
\subsection{Kinematic Map}
The kinematic map is given by
\begin{equation}
\label{KM}
{\bf \dot{q}} = {\bf G}({\bf q}) \bm{\nu}
\end{equation}
where $\bf G$ is the matrix mapping the generalized velocity of the body to the time derivative of the position and the orientation ${\bf \dot{q}}$.

\begin{figure}
\centering
\includegraphics[width=3in]{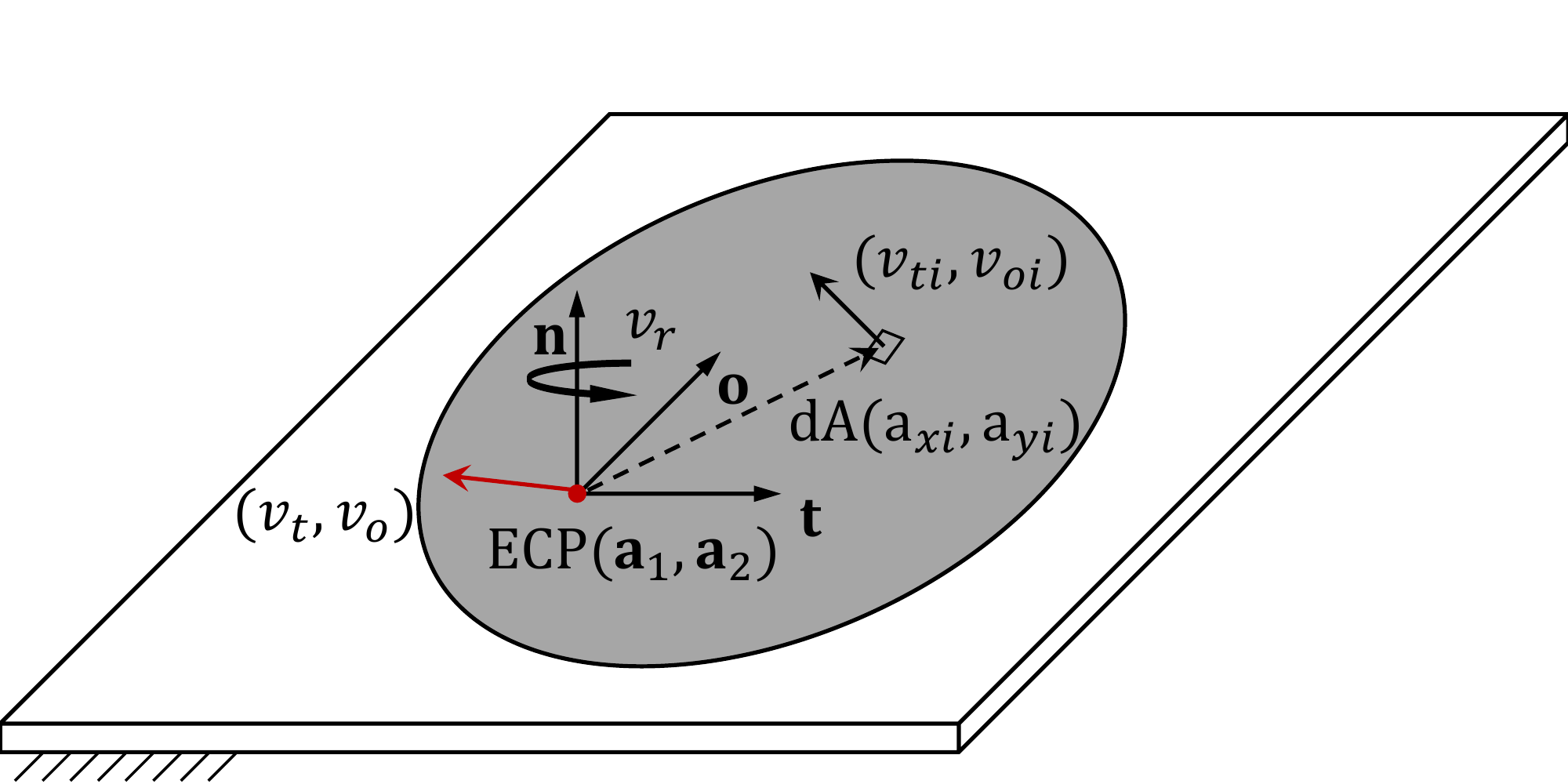}
\caption{Notation for planar contact patch. } 
\label{figure_pressure} 
\end{figure} 

\subsection{Friction Model}
\label{sec:fric}
%NC change back to red color

We use a friction model based on the maximum power dissipation principle, which has been previously proposed in the literature for point contact~\cite{Moreau1988}. The maximum power dissipation principle states that among all the possible contact wrenches (i.e., forces and moments) 
%that lie within the friction ellipsoid, 
the wrench that maximize the power dissipation at the contact are selected. For non-point contact, we will use a generalization of the maximum power dissipation principle, where, we select contact wrenches and contact velocities that maximize the power dissipation over the entire contact patch. We will now show that the problem formulation using the power loss over the whole contact patch can be reduced to the friction model for point contact with the ECP as the chosen point. Mathematically, the power dissipated over the entire surface, $P_c$ is 
\begin{equation}
\label{eq:power}
P_c = -\int_{A} (v_{ti}\beta_{ti}+v_{oi}\beta_{oi})dA
\end{equation}
where $v_{ti},v_{oi}$ are the sliding velocity at $dA$, $\beta_{ti},\beta_{oi}$ are the frictional force per unit area at $dA$. Since we assume the contact patch to be planar, the contact normal is the same at all points on the contact patch. 
%As shown in Figure~\ref{figure_pressure}, the angular velocity is constant across the patch, i.e., $v_{ri} = v_r$, for all $i$. 
Let $v_t ={\bf W}^{T}_{t}
{\bm \nu} $ and $v_o={\bf W}^{T}_{o}
\bm{\nu}$ be the components of tangential velocities at the ECP. Similarly, the angular velocity about contact normal $v_r ={\bf W}^{T}_{r}
\bm{\nu} $.  From basic kinematics, we know that $v_{ti} = v_t - v_{r}(a_{yi}-a_y)$ and $v_{oi} = v_o + v_{r}(a_{xi}-a_x)$, where ($a_x$, $a_y$) are the $x$ and $y$ coordinates of the ECP and ($a_{xi}$, $a_{yi}$) are the $x$ and $y$ coordinates of a point on the patch. Substituting the above in Equation~\eqref{eq:power} and simplifying, we obtain
\begin{equation}
\label{eq:power2}
P_c = -\left[\int_{A} v_{t}\beta_{ti}dA + \int_{A}v_{o}\beta_{oi}dA +\int_{A}v_{ri}\beta_{ri}^{\prime}dA \right ]
\end{equation}
where $\beta^{\prime}_{ri} =  -\beta_{ti}(a_{yi}-a_y) + \beta_{oi}(a_{xi}-a_x)$.
%$\beta_{ri}$ is the equivalent frictional moment of $dA$ at ECP.
%
% Based on basic kinematics:
%
% \begin{equation}
% \begin{aligned}
% &v_{ti} = v_t - v_{ri}(a_{yi}-a_y)\\
% &v_{oi} = v_o + v_{ri}(a_{xi}-a_x)\\
% &v_{ri} = w_z \\
% &\beta^{\prime}_{ri} = \beta_{ri} +\beta_{ti}(a_{yi}-a_y) - \beta_{oi}(a_{xi}-a_x)
% \end{aligned}
% \end{equation}
% where $\beta_{ri}$ represents frictional moment at $dA$. After substituting, we get:
% \begin{equation}
% \int v_{t}\beta_{ti}+v_{o}\beta_{oi}+w_z\beta_{ri}dA
% \end{equation}
By noting that $\int \beta_{ti}dA = \lambda_t, \int \beta_{oi}dA  = \lambda_o, \int \beta^{\prime}_{ri}dA  = \lambda_r$, where $\lambda_t$, $\lambda_o$ are the net tangential forces at the ECP and $\lambda_r$ is the net moment about the axis normal to the contact patch and passing through the ECP, the power dissipation over the entire contact patch is given by $P_c = - (v_t \lambda_t + v_o\lambda_o + v_r \lambda_r)$. 

For specifying a friction model, we also need a law or relationship that bounds the magnitude of the friction forces and moments in terms of the magnitude of the normal force~\cite{Goyal1991}. Here, we use an ellipsoidal model for bounding the magnitude of tangential friction force and friction moment. This friction model has been previously proposed in the literature~\cite{Goyal1991,Moreau1988,XieC16,NilanjanChakraborty2007} and has some experimental justification~\cite{Howe1996}.
Thus, the contact wrench is the solution of the following optimization problem: 
\begin{equation}
\begin{aligned}
\label{equation:friction}
\argmax_{\lambda_t, \lambda_o, \lambda_r} \quad -(v_t \lambda_t + v_o\lambda_o + v_r \lambda_r)\\
{\rm s.t.} \quad \left(\frac{\lambda_t}{e_t}\right)^2 + \left(\frac{\lambda_o}{e_o}\right)^2+\left(\frac{\lambda_r}{e_r}\right)^2 - \mu^2 \lambda_n^2 \le 0
\end{aligned}
\end{equation}
where the magnitude of contact force and moment at the ECP, namely, $\lambda_t$, $\lambda_o$, and $\lambda_r$ are the optimization variables. The parameters, $e_t$, $e_o$, and $e_r$ are positive constants defining the friction ellipsoid and $\mu$ is the coefficient of friction at the contact~\cite{Howe1996,Trinkle1997}.
Thus, we can use the contact wrench at the ECP to model the effect of entire distributed contact patch.
%\comment{ Therefore $v_t$ and $v_o$ are the tangential components of velocity at the ECP; $v_r$ is the relative angular velocity about the normal at ECP. } 
%This constraint is the elliptic dry friction condition suggested in~\cite{howe1996practical} based upon evidence from a series of contact experiments. \add{
%Note that, the ellipsoid constraint in our friction model is the constraint on the friction force and moment that acts on the slider during the motion. This friction model has been previously proposed in the literature~\cite{goyal1991planar} and has some experimental justification~\cite{howe1996practical}. 
Note that in the derivation above, there is {\em no assumption made on the nature of the pressure distribution between the two surfaces}.

%} %NC change back

\color{black}
\subsection{Discrete-time Equations}
To discretize the above system of equations, we use a backward Euler time-stepping scheme. Let $t_u$ denote the current time and $h$ be the duration of the time step, the superscript $u$ represents the beginning of the current time and the superscript $u+1$ represents the end of the current time.
Let $\dot{\bm{\nu}} \approx ( {\bm{\nu}}^{u+1} -{\bm{\nu}}^{u} )/h$, and  $\dot{\bf q}\approx( {\bf q}^{u+1} -{\bf q}^{u} )/h$, and write in terms of impulses. The Newton-Euler Equations and kinematic map~\eqref{eq1} and\eqref{KM} become:
\begin{equation}
\begin{aligned}
\label{eq:discrete_dynamics}
&{\bf M}^{u} {\bm{\nu}}^{u+1} = 
{\bf M}^{u}{\bm{\nu}}^{u}+{\bf W}_{n}p^{u+1}_{n}+{\bf W}_{t}p^{u+1}_{t} \\&
+{\bf W}_{o}p^{u+1}_{o}+{\bf W}_{r}p^{u+1}_{r}+{\bf p}^{u}_{app}+{\bf p}^{u+1}_{vp}
\end{aligned}
\end{equation}
\begin{align}
\label{eq:kinematic map}
&{\bf q}^{u+1} ={\bf q}^{u}+h{\bf G}^u \bm{\nu}^{u+1}
\end{align}

Using the Fritz-John optimality conditions of Equation~\eqref{equation:friction}, we can write~\cite{trinkle2001dynamic}:
\begin{equation}
\begin{aligned}
\label{eq:friction}
0&=
e^{2}_{t}\mu p^{u+1}_{n} 
{\bf W}^{T}_{t}
\bm{\nu}^{u+1}+
p^{u+1}_{t}\sigma^{u+1} \\
0&=
e^{2}_{o}\mu p^{u+1}_{n}  
{\bf W}^{T}_{o}
\bm{\nu}^{u+1}+p^{u+1}_{o}\sigma^{u+1} \\
0&=
e^{2}_{r}\mu p^{u+1}_{n}{\bf W}^{T}_{r}
\bm{\nu}^{u+1}+p^{u+1}_{r}\sigma^{u+1}\\
0& \le (\mu p^{u+1}_n)^2- \left(\frac{p^{u+1}_{t}}{e_{t}}\right)^2- \left(\frac{p^{u+1}_{o}}{e_{o}}\right)^2\\&- \left(\frac{p^{u+1}_{r}}{e_{r}}\right)^2 \perp \sigma^{u+1} \ge 0
\end{aligned}
\end{equation}
where $\sigma$ is a Lagrange multiplier corresponding to the inequality constraint in~\eqref{equation:friction}. Note that ${\bf W}_n$, ${\bf W}_t$, ${\bf W}_r$ in Equations~\eqref{eq:discrete_dynamics} and~\eqref{eq:friction} are dependent on the ECP at the end of the time step $u+1$. However, the ECP is not known apriori and changes as the objects move. Therefore, in Section~\ref{sec:modcc}, we provide equations that the ECP must satisfy while ensuring that the rigid body constraints that the two objects cannot interpenetrate are not violated.

%% file: 4_contact_constraints.tex
\section{Modeling Contact Constraints}
\label{sec:modcc}
In complementarity-based formulation of dynamics, the normal contact constraint for each contact is written as 
\begin{equation} \label{equation:normal contact}
0\le \lambda_{n} \perp \psi_{n}(
{\bf q},t) \ge 0
\end{equation}
where $\psi_{n}(
{\bf q},t)$ is the gap function between two objects, say, $F$ and $G$. When $\psi_{n}({\bf q},t) > 0$, the objects are separate, when $\psi_{n}({\bf q},t) = 0$, the objects touch, when $\psi_{n}({\bf q},t) < 0$, the objects penetrate each other. Since $\psi_{n}({\bf q},t)$ usually has no closed form expression, and the contact constraints should be satisfied at the end of the time step,
%NC changed
%{\color{red} (i.e., the rigid body assumption so that there will be no-penetration between objects)},
state-of-the-art time steppers~\cite{NME:NME1049, Liu2005, Stewart2000, BrulsAC2018,SchindlerR+2015, Studer2009, capobiancoE2018} do the following: (a) use a collision detection algorithm to get the closest point at the beginning of the time-step (b) approximate the distance function at the end of the time step using a first order Taylor's series expansion. Thus, the time-steppers are explicit in the geometric information and the collision detection step is decoupled from the dynamics solution step, where the state of the system and the contact wrenches are computed. In~\cite{NilanjanChakraborty2007}, the authors discussed the limitations of such an approach in terms of undesired interpenetration between rigid objects, and introduced a method whereby the geometry of the bodies are included in the equations of motion, so that simulation with no artificial interpenetration can be guaranteed.

 \begin{figure}
\centering
\includegraphics[width=3in]{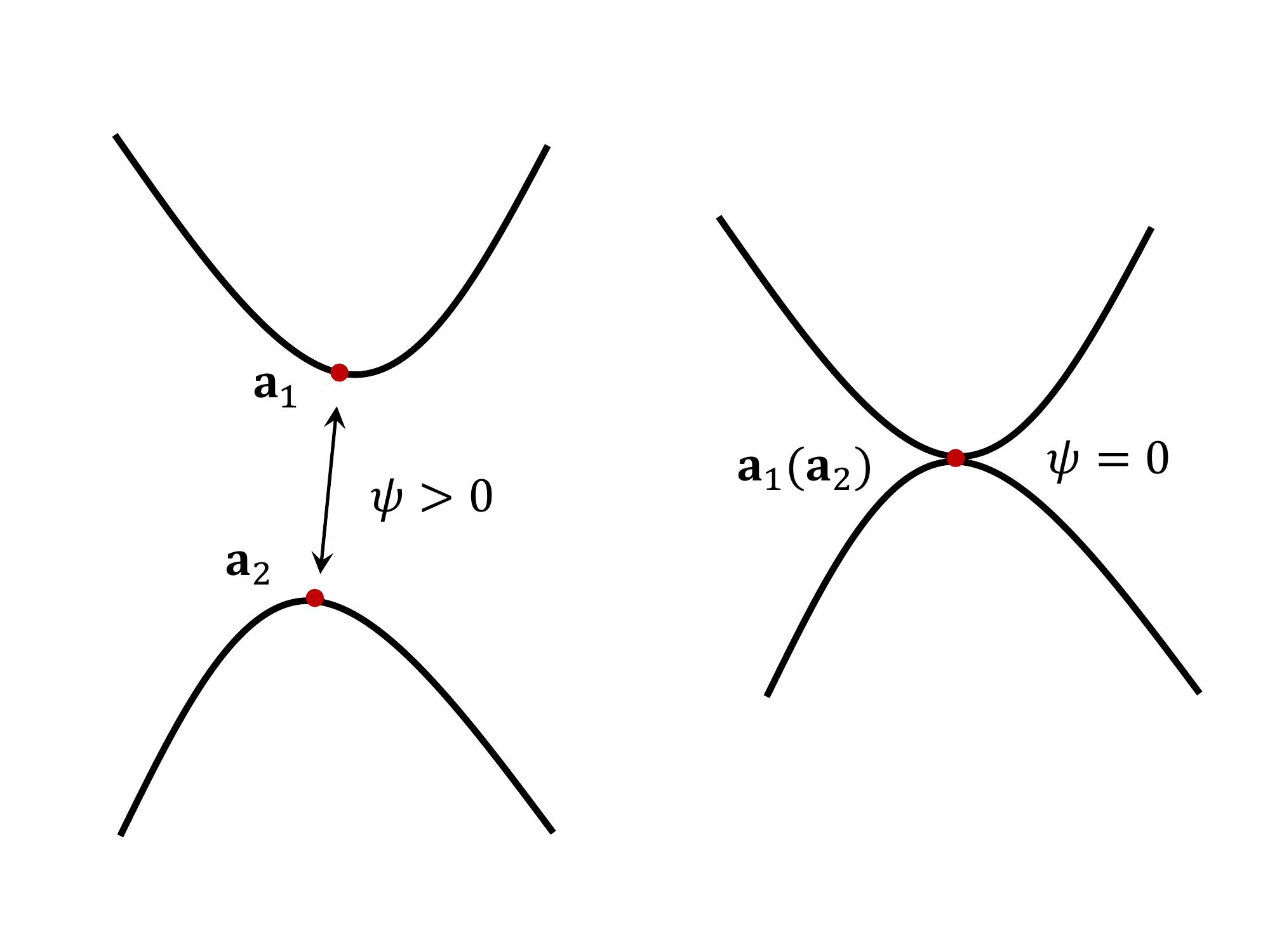}
\caption{Scenario of point contact}
\label{figure:contact_1} 
\end{figure} 

 \begin{figure}
\centering
\includegraphics[width=3in]{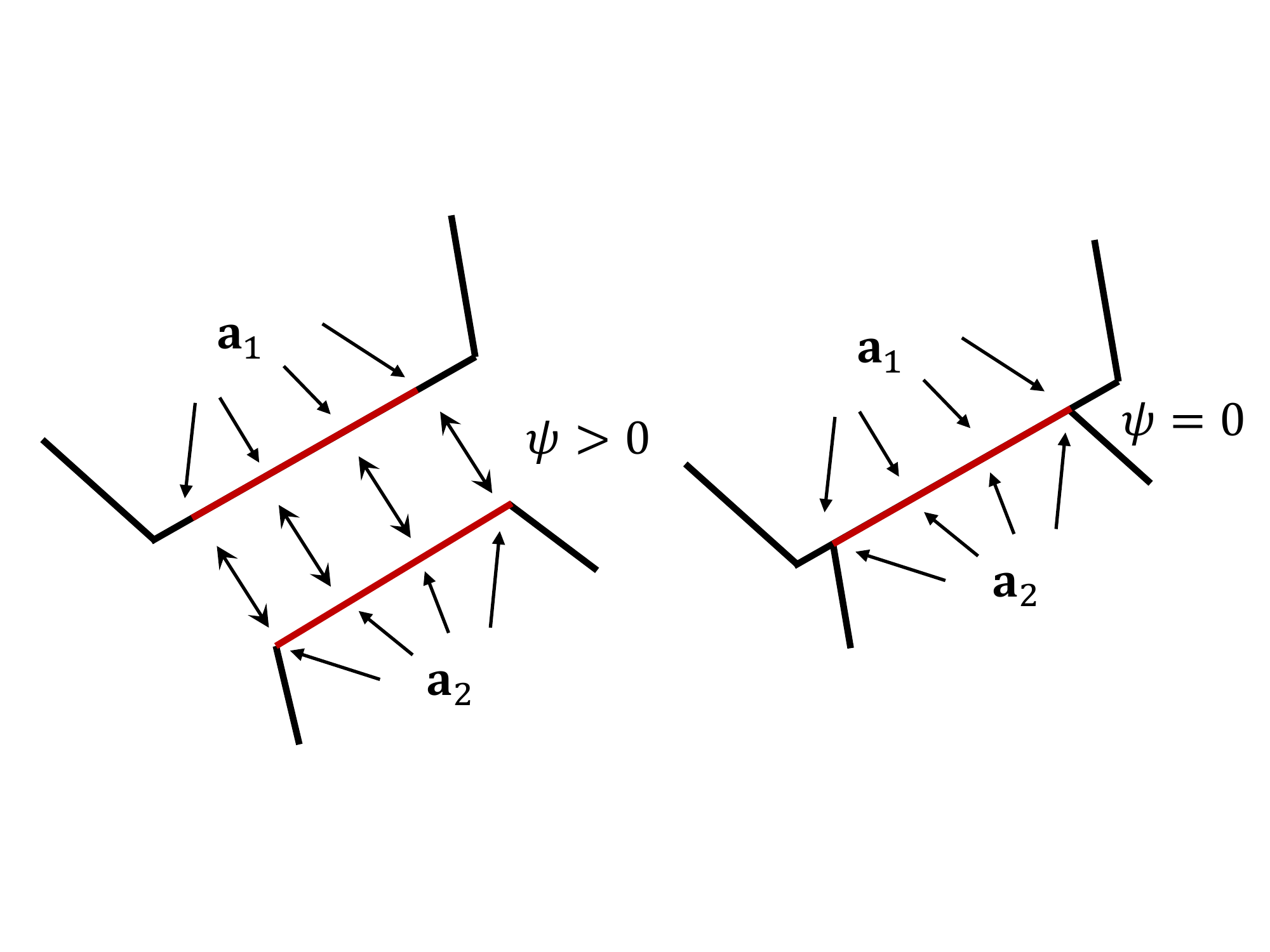}
\caption{Scenario of line contact}
\label{figure:contact_2} 
\end{figure} 

Previous models, including~\cite{NilanjanChakraborty2007}, assume point contact. In Figure \ref{figure:contact_1}, ${\bf a}_1$ and ${\bf a}_2$ are used to represent the closest points on two objects. When $\psi > 0$, the two objects are not in contact, when $\psi = 0$ , the two objects are in contact. In this case, the pair $({\bf a}_1, {\bf a}_2)$ is unique and any collision detection algorithm gives the unique contact information to formulate the dynamic model. However, %apart from point contact, there may also be line or surface contact between two objects.
for line contact or surface contact, as shown in Figure \ref{figure:contact_2}, there can be an
%NC change back {\color{red} an} 
infinite number of closest points. Thus, any point on the contact line (or surface) patch is a valid choice of the closest points and there is no way to choose one (see Figure~\ref{figure:contact_nonpenetration_b}, \ref{figure:contact_nonpenetration_c}). In practice, a number of (potential) points on the (potential) contact patch are chosen arbitrarily, and the non-penetration constraint is enforced at each of the discrete points (e.g., in Figure~\ref{figure:contact_nonpenetration_c}, the points chosen may be the four vertices of the face of $F$ on the face of $G$). However, the contact patch may change as the system evolves, in which case arbitrary {\em a priori} selection of contact points can lead to wrong results (e.g., in Figure~\ref{figure:contact_nonpenetration_c}, the vertex labeled, $v_F$ may not be in contact with $G$ after motion and only part of the face of $F$ may be in contact with $G$).  Thus ad hoc selection of points to enforce contact constraints along with other assumptions like linearization of friction cone can lead to inaccuracies in simulation results (see the Example $1$ in the simulation results section).
%For a priori selected points, it may be possible that the non-penetration conditions are satisfied at the selected points but the two objects still penetrate each other.

The guarantee of non-penetration in~\cite{NilanjanChakraborty2007} is valid for any point on the contact surface. From basic physics, we know that there is a unique point on one contact surface that can be used to model the surface (line) contact as point contact where the integral of the total moment (about the point) due to the distributed normal force on the contact patch is zero. We call this closest point the {\em equivalent contact point} (ECP). This equivalent contact point along with the equivalent contact wrench (due to distributed normal force as well as distributed friction force over the contact patch) that acts at this point so that the two objects do not penetrate is unique. So we can potentially use the implicit time-stepping method to solve for the ECP as the closest point on the contact surface, its associated wrench and configurations of the object simultaneously. In the subsequent sections, we show that it is indeed the case and prove that the guarantee of non-penetration in~\cite{NilanjanChakraborty2007} that is valid for point contact between two objects can be extended to line and surface contact.   
\begin{figure*}%
\centering
\subcaptionbox{ \label{figure:contact_nonpenetration_a}}
{\includegraphics[width=0.66\columnwidth]{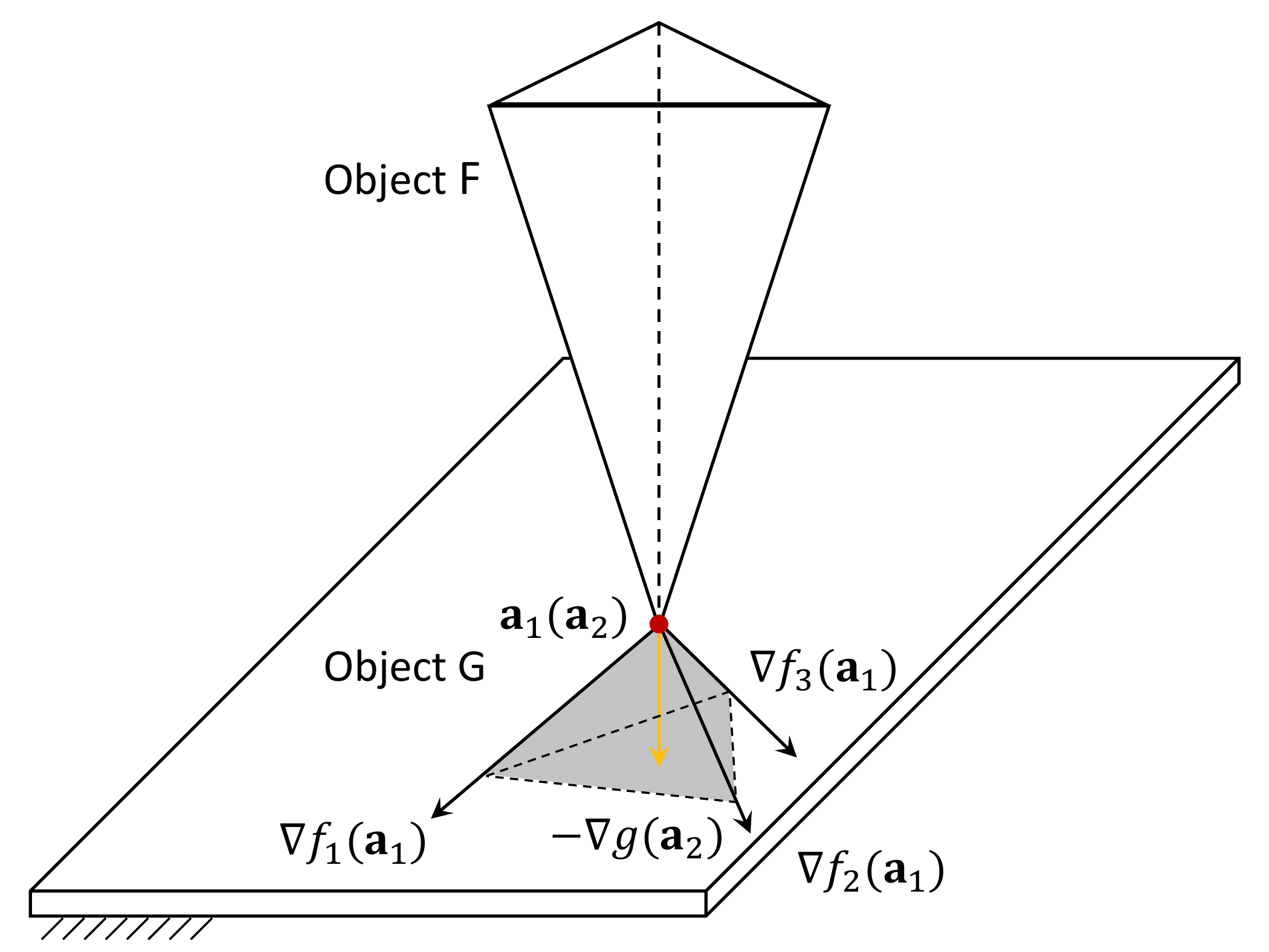}}
\subcaptionbox{ \label{figure:contact_nonpenetration_b}}
{\includegraphics[width=0.66\columnwidth]{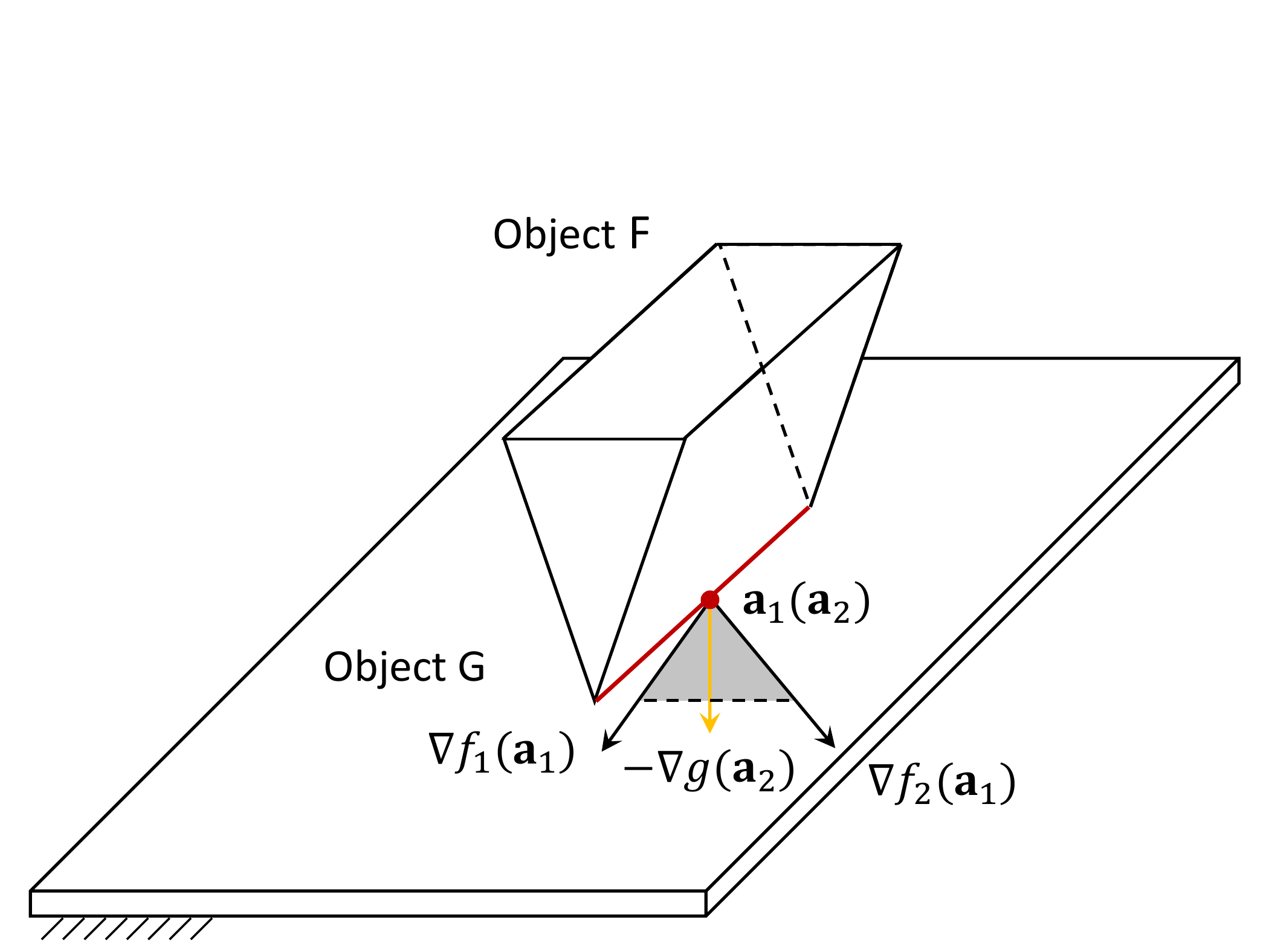}}
\subcaptionbox{ \label{figure:contact_nonpenetration_c}}
{\includegraphics[width=0.66\columnwidth]{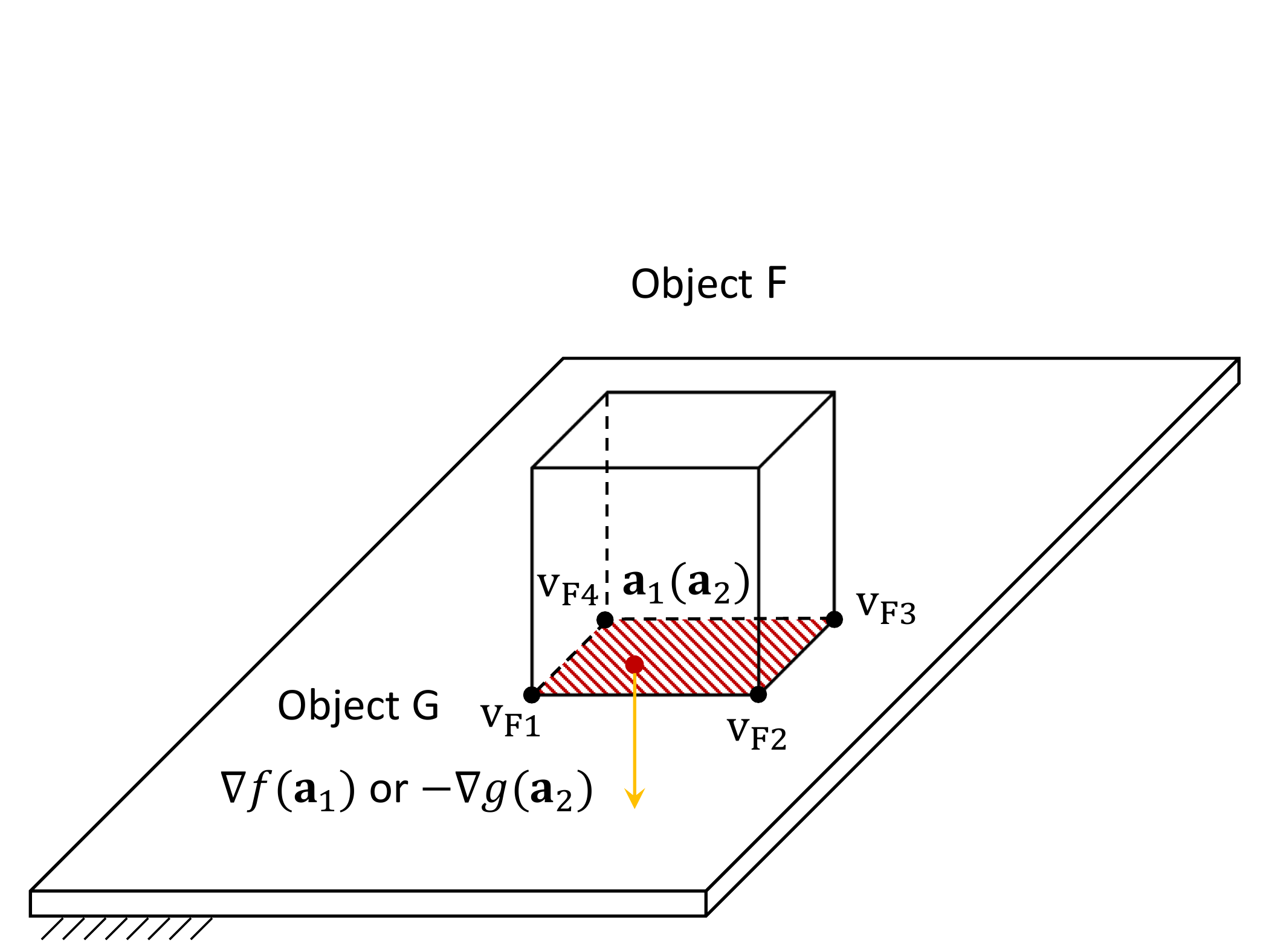}}\\
\subcaptionbox{ \label{figure:contact_nonpenetration_d}}
{\includegraphics[width=0.66\columnwidth]{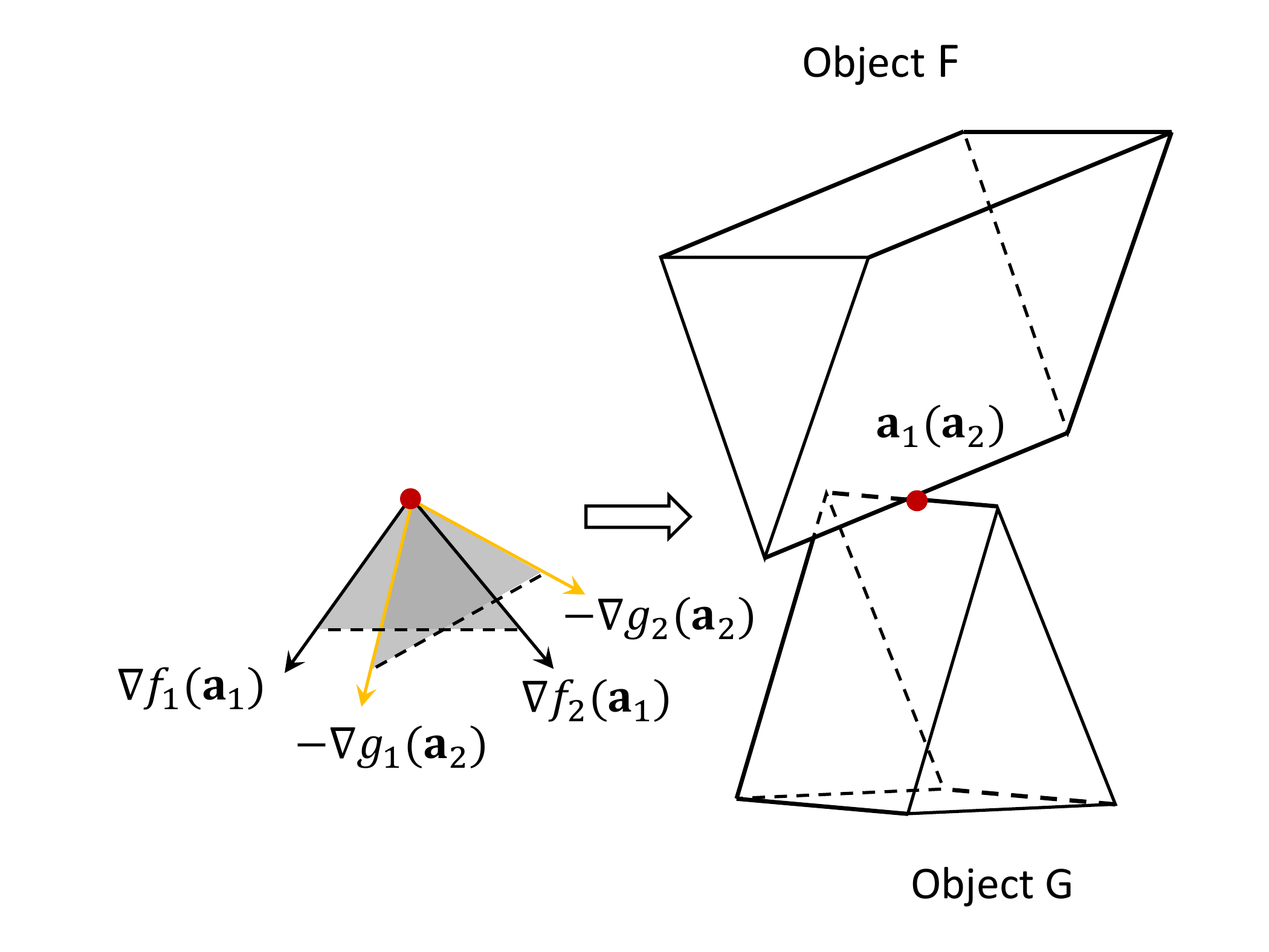}}
\subcaptionbox{ \label{figure:contact_nonpenetration_e}}
{\includegraphics[width=0.65\columnwidth]{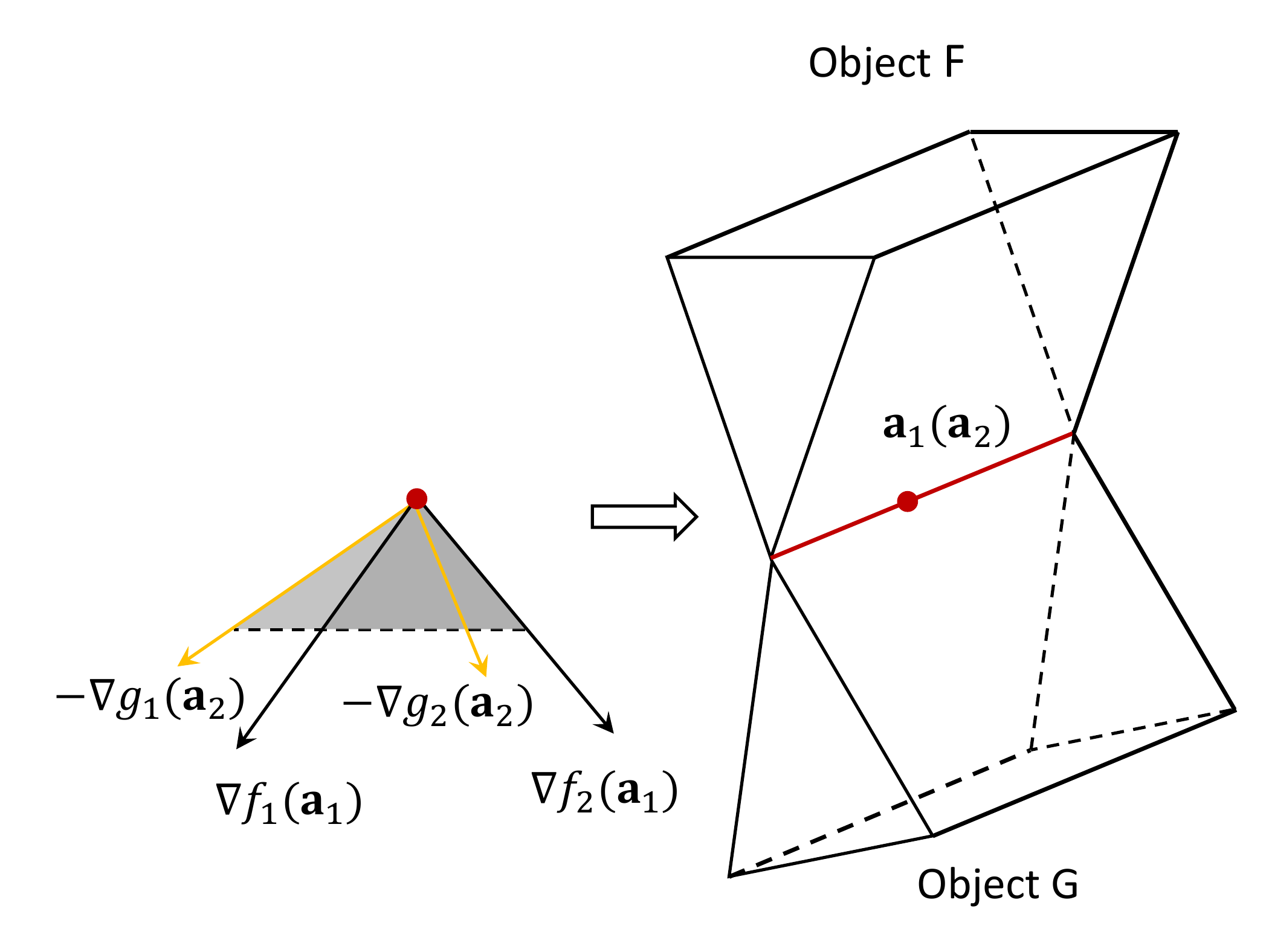}}
\caption{Contact cases:(a) point contact between corner and flat surface, where the supporting hyperplane is uniquely defined by the intersection of pyramidal cone on F and the reversed normal line on $G$. (b) line contact between edge and flat surface, where the supporting hyperplane is uniquely determined by the intersection between planar cone on $F$ and a reversed coplanar normal on $G$. (c) surface contact between two flat surface, where the supporting hyperplane is uniquely determined by two collinear normal lines on $F$ and $G$ (d) point contact between two edges, where the supporting hyperplane is uniquely defined by intersection of two planar cones (e) line contact between two edges, where the supporting hyperplane is determined by two coplanar planar cones, thus the hyperplane is non-unique.  }
\label{figure:contact_nonpenetration_multiple}
\end{figure*}.

%The mathematical concepts required for the formal proofs is described in the Appendix. Note that our objects are modeled as an intersection of differentiable convex functions. However, this does not imply that the normal to a point on the object is uniquely defined (as the point can be at intersections of multiple convex functions). The concept of normal cone is used to represent the set of normal vectors for the points on the contact patch (please see Appendix). Each normal vector defines a supporting hyperplane to the contact patch at that point. The supporting hyperplane is also called separating hyperplane when objects touch each other without penetration. Thus, we use the Separating hyperplane theorem (see Appendix) to explain the conditions which guarantee non-penetration between objects. Finally, we introduce our algebraic contact constraints based on KKT conditions and prove that it prevents penetration when objects have line or surface contact with each other.

\subsection{Algebraic contact constraints based on KKT condition}
Consider two convex objects $F$ and $G$ described by the intersecting convex inequalities. Let $f_i(\bm{\xi}_1) \le 0, i = 1,...,m$ and $g_j(\bm{\xi}_2) \le 0, j=m+1,...,n$, be $\mathcal{C}^2$ convex functions (i.e., twice differentiable with continuous second derivatives) representing two convex objects respectively, where $\bm{\xi}_1$ and $\bm{\xi}_2$ are the coordinates of points on the two objects. Since the closest point is outside the object if it is outside at least one of the intersecting surfaces forming the object, the contact complementarity Equation~\eqref{equation:normal contact} can be written as:
\begin{equation}
\begin{aligned}
\label{equation:contact_multiple_comp}
0 \le p_n \perp \mathop{\rm max}_{i=1,...,m} f_i({\bf a}_2) \ge 0;\\
0 \le p_n \perp \mathop{\rm max}_{j=m+1,...,n}g_j({\bf a}_1) \ge 0
\end{aligned}
\end{equation}
where ${\bf a}_1$ and ${\bf a}_2$ are the closest points on two objects and given by a solution to the following minimization problem:
\begin{equation}
\label{equation:optimazation}
({\bf a}_1,{\bf a}_2) \in {\rm arg} \min_{\bm{\zeta}_1,\bm{\zeta}_2}\{ \|\bm{\zeta}_1-\bm{\zeta}_2 \| \ f_i(\bm{\zeta}_1) \le 0,\ g_j(\bm{\zeta}_2) \le 0 \}
\end{equation}
where $i=1,...,m$ and $j=m+1,...,n$.

Using a slight modification of the Karush-Kuhn-Tucker (KKT) conditions for the optimization problem in Equation~\eqref{equation:optimazation}, the closest points (or ECP) should satisfy the following equations:
\begin{align}
\label{equation:re_contact_multiple_1}
{\bf a}_1-{\bf a}_2 = -l_{k_1}(\nabla f_{k_1}({\bf a}_1)+\sum_{i = 1,i\neq k_1}^m l_i\nabla f_i({\bf a}_1))\\
\label{equation:re_contact_multiple_2}
\nabla f_{k_1}({\bf a}_1)+\sum_{i = 1,i\neq k_1}^m l_i\nabla f_i({\bf a}_1)= -\sum_{j = m+1}^n l_j \nabla g_j ({\bf a}_2)\\
\label{equation:re_contact_multiple_3}
0 \le l_i \perp -f_i({\bf a}_1) \ge 0 \quad i = 1,..,m\\
\label{equation:re_contact_multiple_4}
0 \le l_j \perp -g_j({\bf a}_2) \ge 0 \quad j = m+1,...,n
\end{align}
where $k_1$ represents the index of any one of the active constraints. 
%NC changed {\color{red} Active constraint is a surface of body $F$ where $\bm{a}_1$ lies on.}. 
We will also need an additional complementarity constraint (any one of the two equations in~\eqref{equation:contact_multiple_comp}) to prevent penetration.
\begin{align}
\label{equation:re_contact_multiple_5}
0 \le p_n \perp \mathop{\rm max}_{i=1,...,m} f_i({\bf a}_2) \ge 0
\end{align}

Note that Equations~\eqref{equation:re_contact_multiple_1} to \eqref{equation:re_contact_multiple_4} are not exactly the KKT conditions of the optimization problem in Equation~\eqref{equation:optimazation} but can be derived from the KKT conditions. This derivation is presented in detail in~\cite{NilanjanChakraborty2007} and is therefore omitted here.
% NC changed {\color{red}

\begin{definition}
Let $\bf{x}$ be a point that lies on the boundary of a compact set $F$. Let $\mathbb{I}$ be the index set of active constraints for ${\bf x}$, i.e., $\mathbb{I} = \{i | f_i({\bf x}) = 0, \ i = 1, 2, \dots, n \}$.  The normal cone to $F$ at ${\bf x}$, denoted by $\mathcal{C}(F,{\bf x})$, consists of all vectors in the conic hull of the normals to the surfaces (at ${\bf x}$) represented by the active constraints. Mathematically,
$$\mathcal{C}(F,{\bf x}) = \{ {\bf y} \vert  {\bf y} = \sum_{i \in \mathbb{I}} \beta_i \nabla f_i({\bf x}), \beta_i\ge 0\}$$.
\end{definition}

\begin{definition}
\label{def:hyper}
 Let $F$ be a compact convex set and let ${\bf x}_0$ be a point that lies on the boundary of $F$. Let $\mathcal{C}(F,\bf{x}_0)$ be the normal cone of $F$ at ${\bf x}_0$. The supporting plane of $F$ at ${\bf x}_0$ is a plane passing through ${\bf x}_0$ such that all points in $F$ lie on the same side of the plane. In general, there are infinitely many possible supporting planes at a point. In particular any plane
$ \mathcal{H}({\bf x}) = \{ {\bf x}|  \bm{\alpha}^T({\bf x} - {\bf x}_0) = 0\}$ 
% Such that:
%$$\mathcal{H}(\bf{x}_o)= 0, \forall \bf{x}_o\in \partial F$$ %$$\mathcal{H}(\bf{x})\le \mathcal{H}(\bf{x}_o), \forall \bf{x}\in F$$ 
 where ${\bm{\alpha}} \in \mathcal{C}(F,\bf{x}_0)$ is a supporting plane to $F$ at ${\bf x}_0$.
\end{definition}

\begin{definition}
The touching solution between two objects $F$ and $G$ is for ECPs ${\bf a}_1$ and ${\bf a}_2$  satisfying:
\begin{enumerate}
\item The points ${\bf a}_1$ and ${\bf a}_2$ that satisfy Equations~\eqref{equation:re_contact_multiple_1} to \eqref{equation:re_contact_multiple_5} lie on the boundary of objects $F$ and $G$ respectively. 
\item The objects can not intersect with other.
\end{enumerate}
\end{definition}

\begin{proposition}{Let objects F and G be two nonempty convex sets. Using Equations~\eqref{equation:re_contact_multiple_1} to~\eqref{equation:re_contact_multiple_5} to model line or surface contact between F and G, if the distance between objects is non-negative, then we get the solution for ECPs as the closest points on the boundary of two objects. Moreover, if the distance is zero, then we get only touching solution. }
\end{proposition}

\begin{proof}
First, when objects are separate, Equations~\eqref{equation:re_contact_multiple_1} $\sim$~\eqref{equation:re_contact_multiple_5} will give us the solution for ${\bf a}_1$ and ${\bf a}_2$ as the closet points on the boundary of object\footnote{Although the closest points may be non-unique, we can just pick any point. This does not affect our dynamics since the contact wrench is $0$.}. The proof is same as in ~\cite{NilanjanChakraborty2007}. 

When the distance between two objects is zero, the modified KKT conditions~\eqref{equation:re_contact_multiple_1} to~\eqref{equation:re_contact_multiple_4} will give us the optimal solution for the minimization problem (Equation~\eqref{equation:optimazation}), i.e., ${\bf a}_1 = {\bf a}_2$. Furthermore, Equations~\eqref{equation:re_contact_multiple_1} to~\eqref{equation:re_contact_multiple_4} and Equation~\eqref{equation:re_contact_multiple_5} together give us the solution for ${\bf a}_1$ and ${\bf a}_2$ as the touching solution.

By the definition of the touching solution, ECP ${\bf a}_1$ and ${\bf a}_2$ should lie on the boundary of objects $F$ and $G$ respectively. We prove it by contradiction. If ${\bf a}_1$ lies within the interior of object $F$, then from Equation~\eqref{equation:re_contact_multiple_3}, $f_i({\bf a}_1) < 0,\  l_i = 0, \ \forall i = 1,...,m$. From Equation~\eqref{equation:re_contact_multiple_1}, ${\bf a}_1 = {\bf a}_2$, thus $f_i({\bf a}_2) < 0, \ \forall i = 1,...,m$, which contradicts with Equation~\eqref{equation:re_contact_multiple_5}. Thus ${\bf a}_1$ has to lie on the boundary of object $F$. If ${\bf a}_2$ lies within object $G$, from Equation~\eqref{equation:re_contact_multiple_4}, $g_j({\bf a}_2) < 0$, $l_j = 0,  \ \forall j = m+1,...,n$. Thus, $\sum_{j = m+1}^n l_j \nabla g_j ({\bf a}_2) = 0$. Since the left hand side of Equation~\eqref{equation:re_contact_multiple_2} is nonzero, this leads to a contradiction. Thus ${\bf a}_2$ lies on the boundary of object $G$.

We will now prove that the interior of $F$ and $G$ are disjoint. We prove it based on the supporting hyperplane theorem. Let $\mathcal{H}$ be the supporting plane to $F$ at the point ${\bf{a}_1} \in \partial F$, where the normal ${\bm \alpha} \in \mathcal{C}(\bf{F}_i,\bf{a}_{1}) $. 
%Since $\bf{a}_{1}$ lies on $\mathcal{H}$, which can be mathematically written as ${\bm \alpha}^T\bf{a}_{1}+{\beta} = 0$. Therefore, the constant ${\beta} = - {\bm \alpha}^T\bf{a}_{1}$, and 
The supporting plane is given by $\mathcal{H} = \{{\bf x} | \bm{ \alpha}^T({\bf x}-{\bf a}_{1})= 0 \}$. Since the plane $\mathcal{H}$ supports $ F$ at ${\bf a}_{1}$,  for all points ${\bf x} \in  F$, the affine function $\bm{ \alpha}^T({\bf x}-{\bf a}_{1})\le 0$. In other words, the halfspace $\{{\bf x} | \bm{\alpha}^T({\bf x}-{\bf a}_{1})\le 0 \}$ contains $F$.  Now we need to prove that the halfspace $\{{\bf x} | \bm{\alpha}^T({\bf x}-{\bf a}_{1}) \ge 0 \}$ contains object $G$, which would imply that objects $F$ and $G$ can be separated by $\mathcal{H}$. For point ${\bf a}_2 \in \partial G$, since ${\bf a}_1 = {\bf a}_2$, ${\bf a}_2$ lies in $\mathcal{H}$. For other points ${\bf y} \in \{G \setminus {\bf a}_2\}$, we have $\bm{\alpha}^T({\bf y}-{\bf a}_{1}) = \bm{ \alpha}^T({\bf y}- {\bf a}_{2}+{\bf a}_{2} -{\bf a}_{1}) = \bm{ \alpha}^T({\bf y}-{\bf a}_{2})$. From Equation~\eqref{equation:re_contact_multiple_2}, the direction of normal $\bm{\alpha}$ is  opposite to the normal cone of $G$ at ${\bf a}_2$. Since object $G$ is convex, the projection of the vector ${\bf y}-{\bf a}_{2}$ onto the normal cone at ${\bf a}_2$ is always non-positive. Therefore, the function ${\bm \alpha}^T({\bf y}-{\bf a}_{2})$ is always non-negative. Thus, the halfspace $\{{\bf x} | {\bm \alpha}^T({\bf x}-\bf{a}_{1})\ge 0 \}$ contains object $G$. Thus, we can conclude that the interior of $F$ and $G$ are disjoint.

%In addition, we need to prove that objects $F$ and $G$ are disjoint. We prove it based on the supporting hyperplane theorem. Because $\bm{a}_1$ lies on the boundary of object $F$, thus there exists a supporting hyperplane $\mathcal{H}(\bm{x}) = \bm{a}^T\bm{x}+b$, where the normal  $\bm{a} \in \mathcal{C}(F,\bm{a}_1)$ and constant $b = -\bm{a}^T\bm{a}_1$. Based on  Definition~\ref{def:hyper}, $\mathcal{H}(\bm{x}_o)= 0, \forall \bm{x}_o\in \partial F$, $\mathcal{H}(\bm{x})\le \mathcal{H}(\bm{x}_o), \forall \bm{x}\in F$. Now we need to prove that $\mathcal{H}(\bm{x^{\prime}})\ge \mathcal{H}(\bm{x}_o) = 0, \forall \bm{x}^{\prime}\in \partial G$, i.e., objects $F$ and $G$ are separated by the hyperplane $\mathcal{H}$. From  Equation~\eqref{equation:re_contact_multiple_1}, $l_{k_1}\bm{a} = -(\bm{a}_1-\bm{a}_2)$, $l_{k_1}b = (\bm{a}_1-\bm{a}_2)^T\bm{a}_1$. Therefore, $l_{k_1}\mathcal{H}(\bm{a}_2) = (\bm{a}_1-\bm{a}_2)^T(\bm{a}_1-\bm{a}_2), \forall \bm{a}_2 \in \partial G$. When the non-negative Lagrange multiplier $l_{k_1} = 0$, from Equation~\eqref{equation:re_contact_multiple_1}, $\bm{a}_1 = \bm{a}_2$. Thus, $\mathcal{H}(\bm{a}_2) = 0$. When $l_{k_1} > 0$,  $\mathcal{H}(\bm{a}_2) > 0$. To sum up, $\mathcal{H}(\bm{a}_2) \ge 0, \forall \bm{a}_2 \in \partial G$. This proves our proposition.

\end{proof}

\subsection{Summary of the Dynamic Model} As stated earlier, the full dynamic model consists of the Newton-Euler equations of motion (Equation~\eqref{eq:discrete_dynamics}), the kinematic map between the rigid body generalized velocity and the rate of change of the parameters for representing position and orientation (Equation~\eqref{eq:kinematic map}), the friction model that gives the constraints that the contact wrenches should satisfy (Equation~\eqref{eq:friction}), and the contact model that gives the constraints that the equivalent contact point should satisfy for ensuring no penetration between the objects (Equations~\eqref{equation:re_contact_multiple_1} to \eqref{equation:re_contact_multiple_5}). Note that ${\bf W}_n$, ${\bf W}_t$, ${\bf W}_r$ in Equations~\eqref{eq:discrete_dynamics} and~\eqref{eq:friction} are dependent on the ECP at the end of the time step $u+1$. Furthermore, the unknown impulses, the unknown Lagrange multipliers $l_i$ and the unknown ECP in Equations~\eqref{equation:re_contact_multiple_1} to \eqref{equation:re_contact_multiple_5} are at time $u+1$. Thus, we have a coupled system of algebraic and complementarity equations that we have to solve.

The mixed nonlinear complementarity problem (MNCP) that we solve at each time step consists of Equations~\eqref{eq:discrete_dynamics}, ~\eqref{eq:friction}, and~\eqref{equation:re_contact_multiple_1} to~\eqref{equation:re_contact_multiple_5}. The solution of the MNCP for each time step gives us the linear and angular velocity ${\bm{\nu}}^{u+1}$, contact impulses $p_n^{u+1}, p_t^{u+1}, p_o^{u+1}, p_r^{u+1}$, and equivalent contact points ${\bf a}_1^{u+1}, {\bf a}_2^{u+1}$, at the end of the time-step. The position and orientation of the object is obtained from Equation~\eqref{eq:kinematic map}, after we obtain ${\bm{\nu}}^{u+1}$. 
%NC changed {\color{red} 
We define vectors ${\bf l}_1$ and ${\bf l}_2$ as the concatenated vectors of $l^{u+1}_i$ and $l^{u+1}_j$ respectively. The vector function ${\bf f}({\bf a}^{u+1}_1)$ and ${\bf g}({\bf a}^{u+1}_2)$ represent concatenated vectors of functions defining objects $F$ and $G$ ( $f_i({\bf a}^{u+1}_1)$ and $g_j({\bf a}^{u+1}_2)$) respectively.  The vector of unknowns, ${\bf z} = [{\bf z}_u;{\bf z}_v]$ where the vector for unknowns of equality constraints is ${\bf z}_u = [\bm{\nu}^{u+1}; {\bf a}_1^{u+1}; {\bf a}_2^{u+1}; p_t^{u+1}; p_o^{u+1}; p_r^{u+1}]$ and the vector for unknowns of complementary constraints is ${\bf z}_v = [{\bf l}_1;{\bf l}_2;p_n^{u+1};\sigma^{u+1}]$.
The equality constraints in the mixed NCP are:
\begin{equation}
\begin{aligned}
0 &= -{\bf M}^{u} {\bm{\nu}}^{u+1} +
{\bf M}^{u}{\bm{\nu}}^{u}+{\bf W}_{n}^{u+1}p^{u+1}_{n}+{\bf W}_{t}^{u+1}p^{u+1}_{t} \\&
+{\bf W}_{o}^{u+1}p^{u+1}_{o}+{\bf W}_{r}^{u+1}p^{u+1}_{r}+{\bf p}^{u}_{app}+{\bf p}^{u}_{vp}\\
0& = {\bf a}^{u+1}_1-{\bf a}^{u+1}_2\\&+l^{u+1}_{k_1} \left( \nabla f_{k_1}({\bf a}^{u+1}_1)+\sum_{i = 1,i\neq k_1}^m l^{u+1}_i\nabla f_i({\bf a}^{u+1}_1)\right)\\
0&= \nabla f_{k_1}({\bf a}^{u+1}_1)\\&+\sum_{i = 1,i\neq k_1}^m l^{u+1}_i\nabla f_i({\bf a}^{u+1}_1)+\sum_{j = m+1}^n l^{u+1}_j \nabla g_j ({\bf a}^{u+1}_2)\\
0&=
e^{2}_{t}\mu p^{u+1}_{n} 
({\bf W}^T_{t})^{u+1}
\bm{\nu}^{u+1}+
p^{u+1}_{t}\sigma^{u+1}\\
0&=
e^{2}_{o}\mu p^{u+1}_{n}  
({\bf W}^{T}_{o})^{u+1}
\bm{\nu}^{u+1}+p^{u+1}_{o}\sigma^{u+1}\\
0&=
e^{2}_{r}\mu p^{u+1}_{n}({\bf W}^{T}_{r})^{u+1}
\bm{\nu}^{u+1}+p^{u+1}_{r}\sigma^{u+1}
\end{aligned}
\end{equation}

The complementary constraints for $\bm{z}_v$ are:
\begin{equation}
\begin{aligned}
0 \le \left[\begin{matrix}
{\bf l}_1\\
{\bf l}_2\\
p^{u+1}_n\\
\sigma^{u+1}
\end{matrix}\right]\perp 
\left[\begin{matrix}
-{\bf f}({\bf a}^{u+1}_1)\\
-{\bf g}({\bf a}^{u+1}_2)\\
\max\limits_{i=1,...,m} f_i({\bf a}^{u+1}_2) \\
\zeta
\end{matrix}\right] \ge 0
\end{aligned}    
\end{equation}
where $\zeta =(\mu p^{u+1}_n)^2- (p^{u+1}_{t}/e_{t})^2- (p^{u+1}_{o}/e_{o})^2- (p^{u+1}_{r}/e_{r})^2 $. Note that, we assume the contact between two bodies is a single convex contact patch. Therefore, vector function ${\bf f}({\bf x})$, which describes the boundary of object $F$, can be simplified as a single function.

Although, for simplicity of exposition, we have written the equations for one object with a single (potentially non-point) contact with another object, they can be generalized to multiple bodies by concatenating the system of equations for each body. 
%NC changed }

\begin{remark}
Please note that we are not making any claims about computing the force distribution over the contact patch. Our claim is that we do not need to know the pressure distribution over the contact patch to compute the dynamics given that the assumptions of rigid body is valid.
Force balance and moment balance equations allow us to solve for the state, ECP, and contact impulses simultaneously without knowing the force distribution, provided, we add additional equations that model the non-penetration constraints of two rigid bodies.
\end{remark}

% In the next section we use these equations to derive closed form expression for the states at the end of the time step for pure translation and rotation when objects are in line or surface contact.

%% file: 5_well_posed.tex
\section{Computing ECP and Equivalent Contact Wrench is Well-posed}
\label{sec:ecp}
%There exists two objects $F$ and $G$, and without loss of generality, we assume $G$ is fixed. Between $F$ and $G$, there can be either non-contact or in contact. When $F$ does not have contact with $G$, since the contact wrench is zero, any point on the contact patch is a valid solution (and will not affect the dynamics), so we ignore this case. As shown in figure~\ref{figure:contact_nonpenetration_multiple}, when $F$ has contact with $G$, the contact patch between objects can be point, line or surface. For point contact, the ECP on contact patch is unique, and we need to prove the uniqueness of the contact wrenches. For line or surface contact, any point on the contact patch satisfies the equations~\ref{equation:re_contact_multiple_1}$\sim$~\ref{equation:re_contact_multiple_5}. Thus there can exist infinite number of possible solutions for the ECP. 

In this section, we show that by solving for the ECP along with the numerical integration of the equations of motion, i.e., by solving a time-stepping problem where the contact constraints (Equations~\eqref{equation:re_contact_multiple_1} to \eqref{equation:re_contact_multiple_5}) are embedded with the dynamic time-step equations and friction model, computing ECP and contact wrench becomes a well-posed problem. Embedding the contact constraints within the dynamic time-stepping scheme implies that there is no separate call to a collision detection software, i.e., both the collision detection and integration of equations of motion are done simultaneously. %so that there is no interpenetration between the bodies at the end of the time step. 
More formally, we prove that { \em given the state of the object at the end of the time-step, the ECP and net contact wrench is unique}. This provides another theoretical justification for embedding the collision detection problem within the dynamic time-step, since if the two are separate, then for patch contacts, the collision detection problems are always ill-posed.

Consider two objects $F$ and $G$, and without loss of generality, we assume $G$ is fixed. Let ${\bf a}_1$ and ${\bf a}_2$ be the position vector of ECP of objects $F$ and $G$, respectively. When $F$ and $G$ are not in contact, since the contact wrench is zero, even if there are multiple pair of closest points, any pair of points that satisfy the equations is a valid solution for ECP (and will not affect the dynamics), so we ignore this case. 

When objects are in contact, as Figure~\ref{figure:contact_nonpenetration_multiple} shows, the contact patch between them can be point contact (Figure~\ref{figure:contact_nonpenetration_a} and~\ref{figure:contact_nonpenetration_d}), line contact (Figure~\ref{figure:contact_nonpenetration_b} and~\ref{figure:contact_nonpenetration_e}) and surface contact (Figure~\ref{figure:contact_nonpenetration_c}). Any pair of points on the two objects that are in the contact patch and have identical coordinates in the inertial frame satisfy the contact constraints. Thus, in line or surface contact case, there exists multiple possible solutions for the closest points ${\bf a}_1$ and ${\bf a}_2$, when the collision detection and dynamic time-stepping problems are decoupled. 

Furthermore, the equations of contact constraints (Equations~\eqref{equation:re_contact_multiple_1} to \eqref{equation:re_contact_multiple_4}) can be simplified as equation of a supporting hyperplane that contains the contact patch.  By the separating hyperplane theorem, since the objects do not penetrate each other, there always exists a supporting hyperplane  (Figure~\ref{figure:contact_nonpenetration_multiple} illustrates this pictorially). The existence of the supporting hyperplane ensures that there exists a contact frame on the contact patch, which can be chosen as given below. Let the ECP ${\bf a}_2$ be the origin of the contact frame. Let the normal axis of the contact frame be the unit vector ${\bf n}\in \mathbb{R}^3$, which is orthogonal to the contact patch, and can be chosen to be the normal to the supporting hyperplane.  The tangential axes of the frame are unit vectors ${\bf t} \in \mathbb{R}^3 $ and ${\bf o}\in \mathbb{R}^3$ that can be chosen on the supporting hyperplane. 
%In addition, we can get the contact wrenches from Equation~\eqref{equation:wrenches}. 

%The specified constraint is the equation of supporting hyperplane. Thus, we choose normal axis $\bm{n}$ to be the normal vector of the plane. 
When objects are in contact, ${\bf a}_1$ coincides with ${\bf a}_2$. Because ECP ${\bf a}_1$ or ${\bf a}_2$ lie in the contact patch, so they should satisfy the equation of the hyperplane. Thus, based on Definition~\ref{def:hyper}, the contact constraints~\eqref{equation:re_contact_multiple_1} to \eqref{equation:re_contact_multiple_5} can be simplified as the equation of plane:
\begin{equation}
\label{Equation::simplified constraints}
{\bf n}\cdot{\bf a}_2 = d 
\end{equation}
where $d$ is a constant.
%where the constant $d$ is distance of the plane from the origin. 

We will now use the simplified contact constraints in~\eqref{Equation::simplified constraints} along with the dynamic time-step equations in~\eqref{eq:discrete_dynamics}
%and friction model (Equation~\eqref{equation:friction}) 
to show that, when objects with given state are in contact, there is an unique solution for the ECP and contact wrenches between $F$ and $G$.

Let $\mathcal{F}_w$ be the inertial frame.
%Let $\{{\bf W}\}=[{\bf X},{\bf Y},{\bf Z}]$ be the inertial frame.
The state of $F$ is described by position and orientation vector ${\bf q} = [q_x, q_y, q_z,{^s\theta_x}, {^s\theta_y}, {^s\theta_z}]^T$ and generalized velocity of center of mass $\bm{\nu} = [v_x , v_y , v_z, {^sw_x}, {^sw_y}, {^sw_z}]^T$, which is described in $\mathcal{F}_w$. %the inertial frame. 
Because ${\bf a}_1 = {\bf a}_2$, the vector from center of gravity of $F$ to the ECP can be defined as ${\bf r} = [a_{2x} - q_x, a_{2y} - q_y, a_{2z} - q_z]^T$. The vector of external impulses is ${\bf p}_{app} =  [p_x, p_y,-m\beta h+p_z,p_{x\tau},p_{y\tau},p_{z\tau}]^T$ where $\beta$ is the acceleration due to gravity. $p_x$, $p_y$ and $p_z$ are the external impulses except the gravity impulse acting on the objects in the directions of spatial frame's $\bf{X}$, $\bf{Y}$ and $\bf{Z}$ axis. $p_{x\tau}$, $p_{y\tau}$ and $p_{z\tau}$ are the external moments in spatial frame's $\bf{X}$, $\bf{Y}$ and $\bf{Z}$ axis.

\begin{proposition}{When two objects with given state at the end of the time-step are in contact, the solutions of ECP and contact wrench are unique and given by the following equations formed by Equation~\eqref{eq:discrete_dynamics} and Equation~\eqref{Equation::simplified constraints}.}
\end{proposition}

\begin{equation}
\begin{aligned}
\label{equation:Euler_123}
 0 &= -\bar{{\bf M}}^{u}({\bf v}^{u+1}-{\bf v}^{u}) + p_n^{u+1}{\bf n}^{u+1}
 \\&+ p_t^{u+1}{\bf t}^{u+1}+p_o^{u+1}{\bf o}^{u+1} +
\left[\begin{matrix}
 p_x^{u}\\
 p_y^{u}\\
 -m\beta h+p_z^{u}
 \end{matrix}\right] 
 \end{aligned}
\end{equation}
 \begin{equation}
\begin{aligned}
\label{equation:Euler_456}
 0  &=- {\bf I}^{u} ({^s{\bf w}}^{u+1}-{^s{\bf w}^{u}})-{^s{\bf w}}^{u+1} \times( {\bf I}^{u} {^s{\bf w}^{u+1}}) \\
&+ p_n^{u+1}{\bf r}^{u+1}\times{\bf n}^{u+1} 
 + p_t^{u+1}{\bf r}^{u+1}\times{\bf t}^{u+1}\\
& + p_o^{u+1}{\bf r}^{u+1}\times{\bf o}^{u+1}
+ p_r^{u+1}{\bf n}^{u+1}
+\left[\begin{matrix}
p_{x\tau}^{u}\\
p_{y\tau}^{u}\\
p_{z\tau}^{u}
\end{matrix}\right]
\end{aligned}
\end{equation}
\begin{equation}
\label{equation:Euler_7}
{\bf n}^{u+1}\cdot{\bf a}_2^{u+1} = d^{u+1} 
\end{equation}
where $\bar{{\bf M}} = {\bf M}(1:3,1:3)$, ${{\bf I}} = {\bf M}(4:6,4:6)$, ${\bf r} ={\bf a}_2 - \bar{{\bf q}}$ and $\bar{{\bf q}} = [q_x,q_y,q_z]^T$.
\begin{proof}
The proof can be separated into two parts:
\begin{enumerate}
\item Contact wrench $p_t^{u+1}$, $p_o^{u+1}$ and $p_n^{u+1}$ are unique.
\item ECP ${\bf a}_2^{u+1}$ and contact wrench $p_r^{u+1}$ are unique.
\end{enumerate}
(1) First, let us prove that the contact wrench $p_t^{u+1}$, $p_o^{u+1}$ and $p_n^{u+1}$ are unique. Since the generalized velocity $\bm{\nu}^{u+1}$ is known at time step $u+1$, Equation~\eqref{equation:Euler_123} can be modified as:
\begin{equation}
\begin{aligned}
\label{eq:22}
{\bf C}
 \left[\begin{matrix}
 p_n^{u+1}\\
 p_t^{u+1}\\
p_o^{u+1}
 \end{matrix}\right] 
 = \bar{{\bf M}}^u({\bf v}^{u+1}-{\bf v}^{u})   -
\left[\begin{matrix}
 p_x^{u}\\
 p_y^{u}\\
 -m\beta h+p_z^{u}
 \end{matrix}\right] 
 \end{aligned}
\end{equation}
where ${\bf C} = \left[\begin{matrix} 
 {\bf n}^{u+1},{\bf t}^{u+1},{\bf o}^{u+1}
 \end{matrix}\right]$. 
 %NC changed {\color{red} 
 Note that, the columns of matrix ${\bf C}$, namely, ${\bf n}^{u+1},{\bf t}^{u+1},{\bf o}^{u+1}$ are the axes of contact frame for the contact between objects. Thus, ${\bf C}$ is an orthogonal matrix ${\bf C}$ and ${\bf C}^T$ is the inverse of ${\bf C}$. Thus, from Equation~\eqref{eq:22}, the values of $p_t^{u+1}, p_o^{u+1}$ and $p_n^{u+1}$ are unique.
 %NC }
 
(2) Now, we will prove that the ECP ${\bf a}_2^{u+1}$ and contact wrench $p_r^{u+1}$ are unique. 
Equation~\eqref{equation:Euler_7} can be written as:
${\bf n}^{u+1}\cdot{\bf r}^{u+1} = d^{u+1}-{\bf n}^{u+1}\cdot\bar{{\bf q}}^{u+1} $
Where $d^{u+1}-{\bf n}^{u+1}\cdot\bar{\bf q}^{u+1}$ is known.

Let
\begin{align}
\label{equation:A}
 {\bf A} &= p_n^{u+1}{\bf n}^{u+1}+ p_t^{u+1}{\bf t}^{u+1}+ p_o^{u+1}{\bf o}^{u+1}\\
 {\bf B} &= {\bf I}^u ({^s{\bf w}}^{u+1}-{^s{\bf w}}^{u})+{^s{\bf w}}^{u+1} \times {\bf I}^u ({^s{\bf w}^{u+1}})-\left[\begin{matrix}
{p_{x\tau}}^{u}\\
{p_{y\tau}}^{u}\\
{p_{z\tau}}^{u}\end{matrix}\right] 
\end{align}

Thus, Equation~\eqref{equation:Euler_456} can be rewritten as:
\begin{equation}
\label{equation:unique_2}
\left[\begin{matrix}[{\bf A}]_{\times}^{T},{\bf n}^{u+1} \\
({\bf n}^T)^{u+1}, 0 \end{matrix}\right] \left[\begin{matrix}\bm{r}^{u+1}\\ p_r^{u+1}\end{matrix}\right]  = \left[\begin{matrix}{\bf B}\\ d^{u+1}-{\bf n}^{u+1}\cdot\bar{\bf q}^{u+1} \end{matrix}\right]
\end{equation}
where $[{\bf A}]_{\times}$ is a skew-symmetric matrix given by:$$[{\bf A}]_{\times} = \left[\begin{matrix}
0 & -A_3 & A_2\\
A_3 & 0 & -A_1\\
-A_2 & A_1 & 0
\end{matrix}\right]$$

By direct calculation, the determinant of the matrix ${\bf D} = \left[\begin{matrix}[{\bf A}]_{\times}^{T},({\bf n})^{u+1} \\
({\bf n}^T)^{u+1}, 0 \end{matrix}\right]$ is given by
\begin{equation}
|{\bf D}| = {\bf A}\cdot{\bf n}^{u+1}.
\end{equation}

From Equation~\eqref{equation:A}, because $p_n^{u+1} \neq 0$, the vector ${\bf A}$ will never be orthogonal to the vector ${\bf n}$. In other words, the determinant $|{\bf D}|$ is always non-zero. Thus, matrix ${\bf D}$ is always invertible. Therefore, from Equation~\eqref{equation:unique_2}, 
%since matrix $\bm{D}$ is always invertible, 
the solution for ${\bf r}^{u+1}$ and $p_r^{u+1}$ should be unique. 
\end{proof}

To sum up, the solution for the contact points and contact wrench is unique when object $F$ and $G$ are in contact. Thus, computing the ECP and the equivalent contact wrench by using our method is a well-posed problem. 

\begin{remark}
Note that the proof given above is non-constructive and it is not used to do any computations for simulating motion. Thus, we need only the existence of the supporting hyperplane, we actually do not need to compute it for the proof to be valid. 
\end{remark}
\begin{remark}
In the proof above, we have not used the friction constraints given by Equations~\eqref{equation:friction}. This is because the friction constraints do not change the fact that the ECP and the contact impulses can be solved uniquely. Depending on the values chosen for the generalized velocity, the obtained contact impulses may violate the friction constraints. However, as stated before, we are not giving any computational recipe to compute the ECP and contact wrenches, only a non-constructive proof of the notion of well-posedness. When we solve the whole MNCP numerically, the solutions that we obtain will satisfy all the constraints including friction constraints.
\end{remark}
\begin{remark}
In this section, we are also not making any claims about the existence or uniqueness of solutions (i.e., velocity of the objects, ECPs, and contact impulse) of the overall mixed nonlinear complementarity problem obtained from the equations of motion with contact constraints. The theoretical proof of existence and uniqueness question of solutions to the MNCP corresponding to our dynamic model is an open question.
\end{remark}
\begin{remark}
The results in this section may also be useful for solving estimation problems when the object is sliding over a surface and has non-point contact.
Given the applied force and velocity at the end of time step, we can compute the equivalent contact forces and moments and the ECP. These can be used in contact parameter estimation.
\end{remark}

%end of comment

%% file: 6_result.tex
\section{NUMERICAL RESULTS}
\label{sec:res}
In this section we present examples to validate and demonstrate our technique. 
%against known analytical results and previous approaches. We use the $\bf{PATH}$ complementarity solver to solve the dynamic model. 
We consider two canonical examples. In the first example, we consider the manipulation of a box-shaped object modeled as a cube on a plane. In the second example, we consider the manipulation of a cylindrical object on the plane. In order to focus on the novel aspects in this paper, we do not focus on simulating the robots that are actually doing the manipulation. Instead we just assume that the effect of the robot is to apply a generalized force or generalized velocity on the objects. We use the $\bf{PATH}$ complementarity solver to solve the nonlinear complementarity problem formed at each discrete time step of the dynamic model.

We consider four different simulation scenarios. The first scenario is for a cube undergoing translational motion on a flat plane, where we compare the solution from our general method using PATH with popular open source dynamics solvers, e.g., ODE~\cite{SmithODE} and BULLET~\cite{CouBullet}. Since, in this case we know the analytical solution~\cite{XieC16}, we can use that as the ground truth. In the second example, a cylinder undergoing combined translation and rotation in the plane with line contact  is simulated. The third example shows a cube transitioning between point, line and surface contact. The last example illustrates a manipulation task for a cylindrical object on the plane with obstacles.

\subsection{Scenario 1: The cube sliding on the plane}

\begin{figure*}%
\centering
\begin{subfigure}[t]{0.50\columnwidth}
\includegraphics[width=\columnwidth]{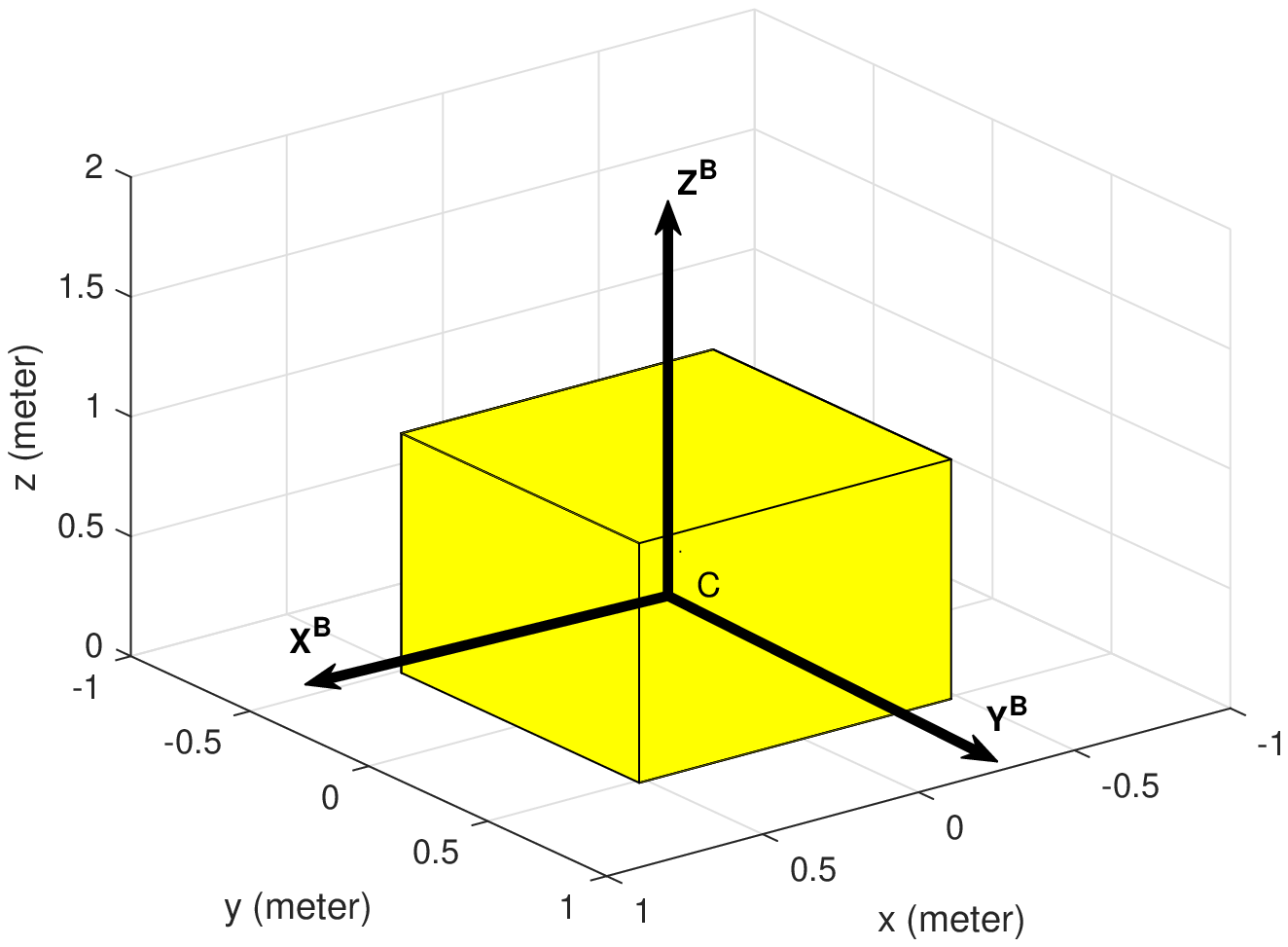}%
\caption{Unit cube on the plane. }
\label{figure:ex2_picture} 
\end{subfigure}\hfill%
\begin{subfigure}[t]{0.50\columnwidth}
\includegraphics[width=\columnwidth]{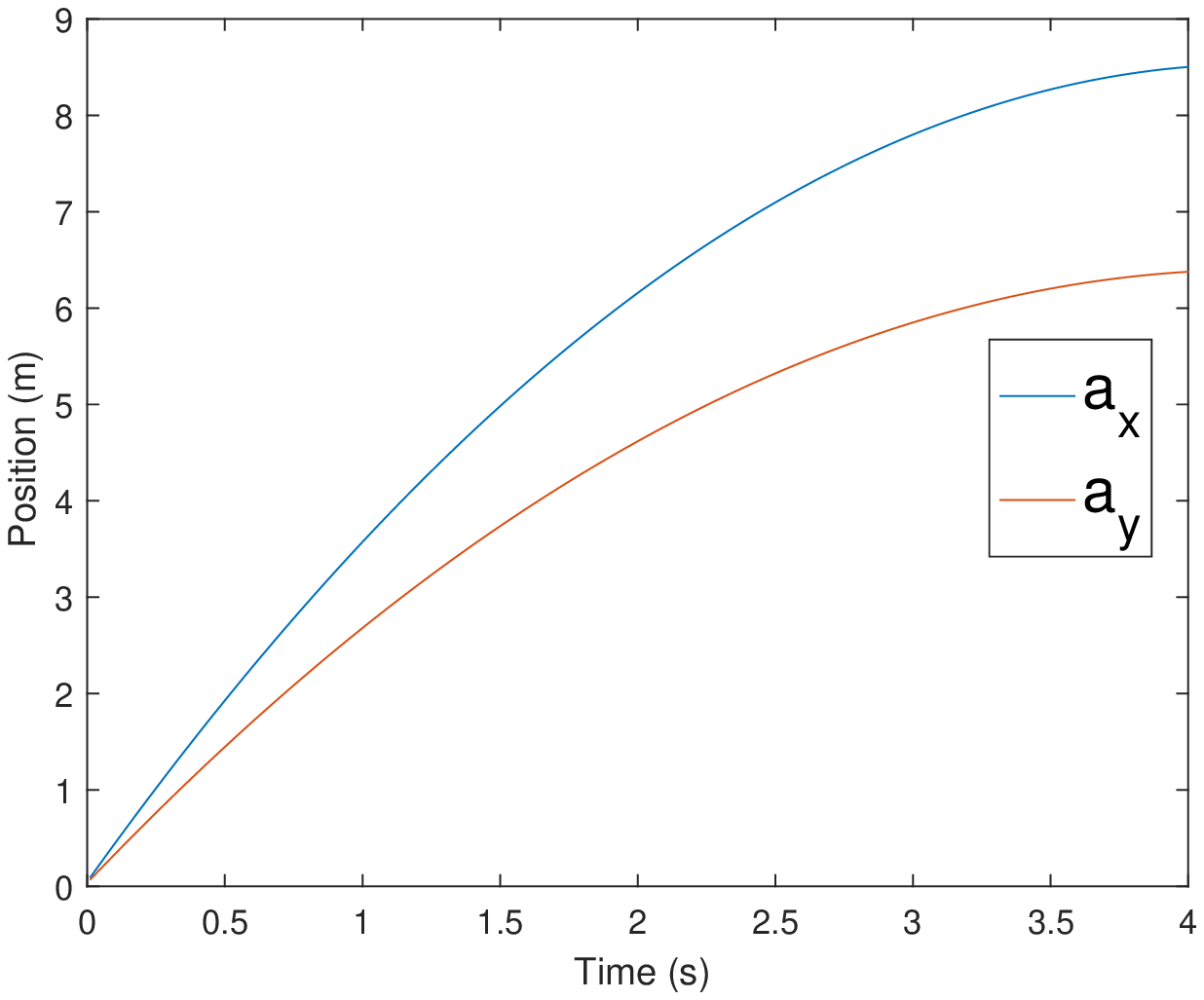}%
\caption{Position components of the ECP $a_x$ and $a_y$ of the cube. }
\label{figure:ex2_closest_point} 
\end{subfigure}\hfill%
\begin{subfigure}[t]{0.50\columnwidth}
\includegraphics[width=\columnwidth]{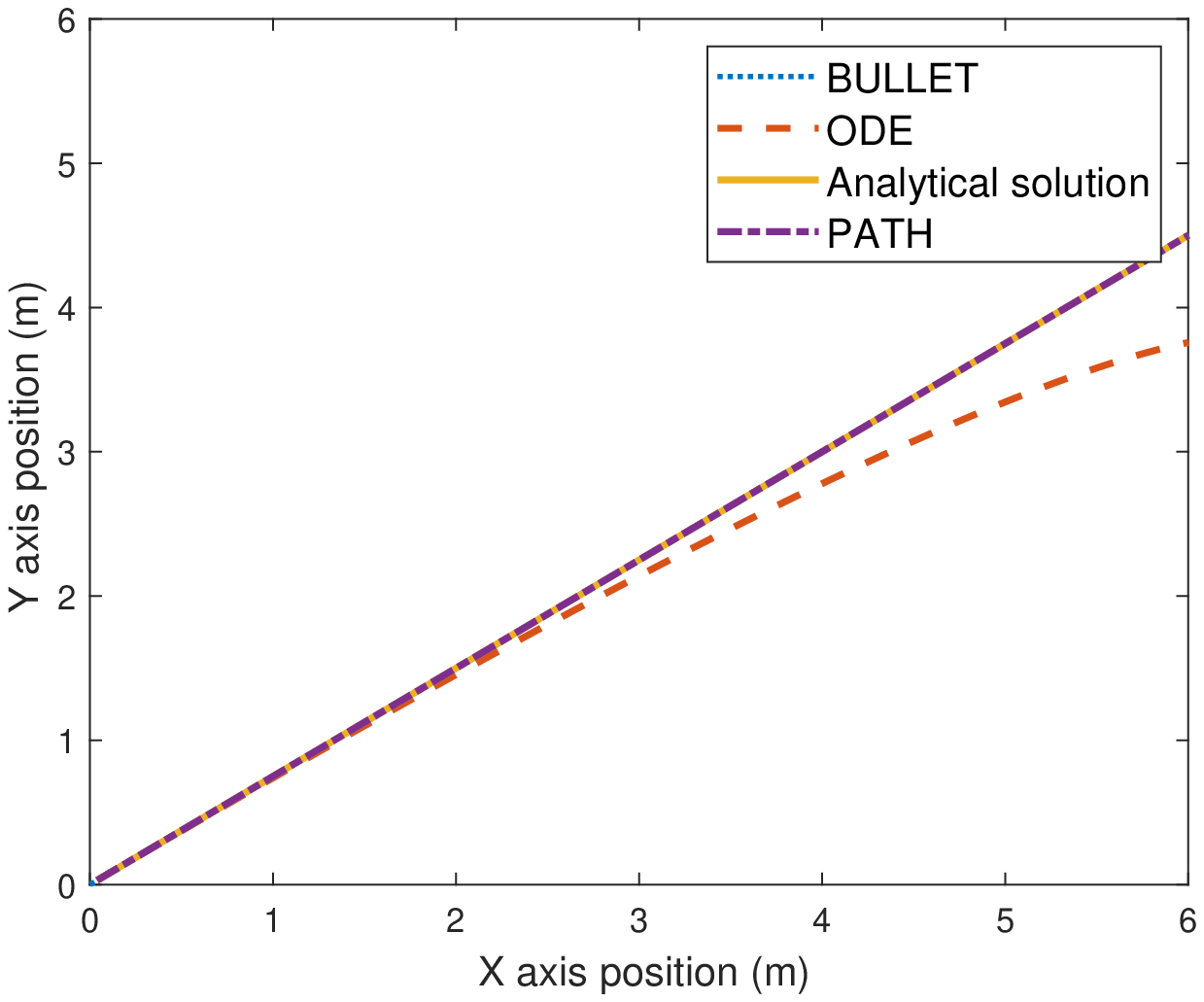}%
\caption{Comparison of analytical solution, our solution, ODE, and BULLET.}
\label{figure:ex2_delta_2} 
\end{subfigure}\hfill%
\begin{subfigure}[t]{0.50\columnwidth}
\includegraphics[width=\columnwidth]{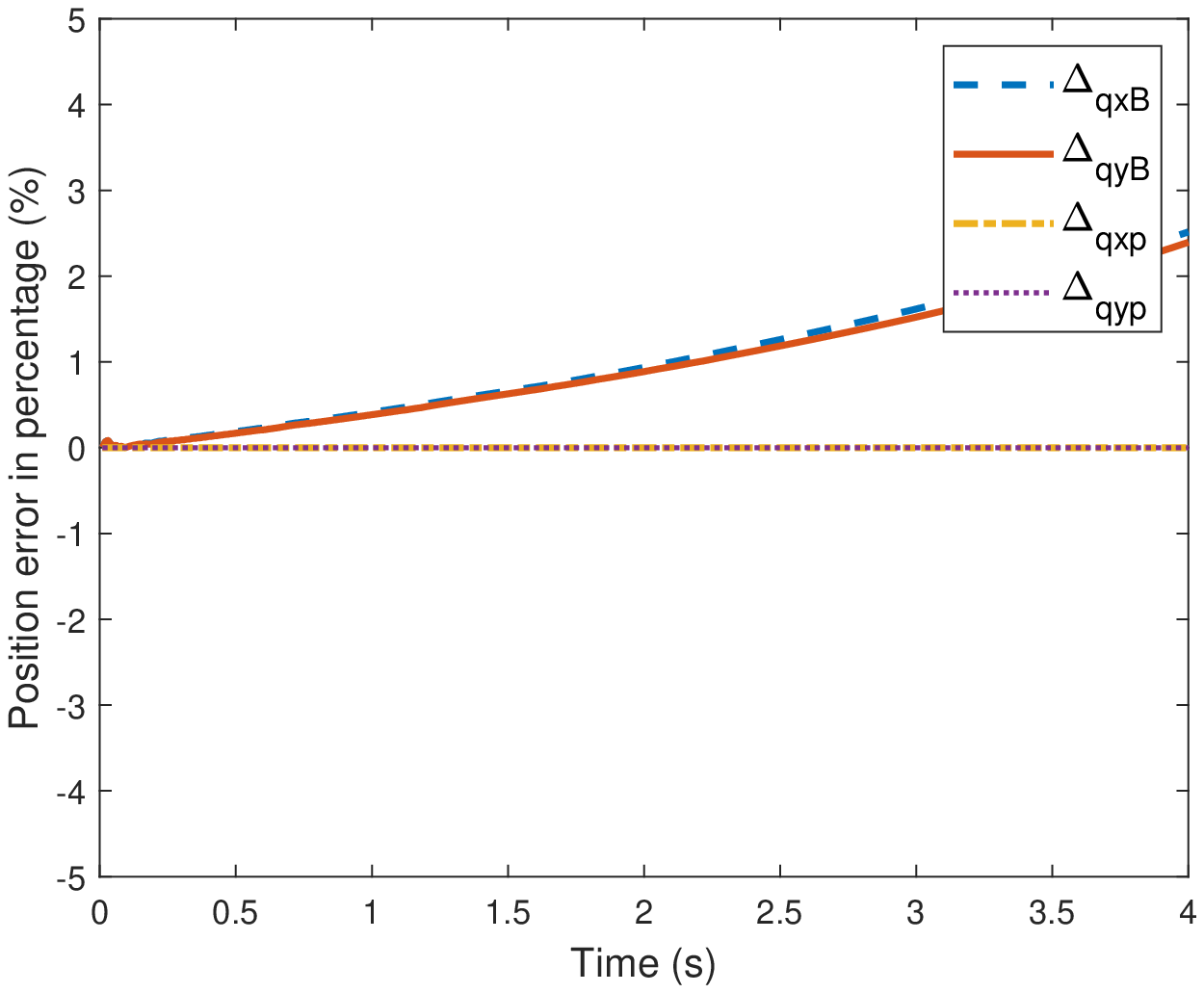}%
\caption{Comparison of analytical solution, our solution and BULLET.}
\label{figure:ex2_delta_3} 
\end{subfigure}%
\caption{Cube sliding on a plane.}
\label{Example1}
\end{figure*}

In this example, we consider a cube with $1$ m side length and $1$ kg weight sliding on the flat plane without rotation (see Figure~\ref{figure:ex2_picture}). Let $ \beta = 9.8$(m/s$^2$), $h = 0.01$s, $\mu = 0.12$, $e_t = 1$, $e_o = 1$, $e_r = 1$m. The initial configuration of the cube is ${\bf q} = [0, 0, 0.5$m$, 1, 0, 0, 0]^T$ (we use the unit quaternion to represent the attitude of the cube), and the generalized velocity is $\bm{\nu} = [4 $  m/s, $ 3  $ m/s $,0, 0, 0,0]^T$. In all three scenarios, we let the acceleration due to gravity is $\beta =9.8$ ($m/s^2$). We choose the time step to be $h = 0.01$s, and the tolerance of simulation to be 1e-8. In this scenario, we set the simulation time to be 4s. 

Figure~\ref{figure:ex2_closest_point} shows the change of equivalent contact point (ECP) obtained by solving the NMCP for each time step using PATH. Note that, for this problem, the ECP, contact wrench and state of the cube at the end of a time step can be computed analytically.

In Figure~\ref{figure:ex2_delta_2}, we compare the analytical solution~\cite{XieC16} (which we consider as the ground truth) to the solution from PATH, the solution obtained from Open Dynamic Engine (ODE), and the solution obtained from BULLET. In ODE, we use $4$ contact points to approximate the contact face and set the contact mode: contact.surface.mode = dContactApprox1. In BULLET, we use 4 contact points to approximate the contact face. In Figure~\ref{figure:ex2_delta_2}, we can see that the solution of ODE has a large error and the direction of sliding changes which can not be true for the given input. The analytical solution and the solution from PATH agrees well depicting the correctness of our approach.

In Figure~\ref{figure:ex2_delta_3}, we compare the analytical solution (which we consider as the ground truth) to the solution from PATH and the solution obtained from BULLET. In PATH, the result describing the position are $q_{x_p}$ , $q_{y_p}$. In BULLET, the result describing the position are $q_{x_B}$ , $q_{y_B}$. The analytical solutions are $q_{x}$, $q_{y}$. We calculate the error in percentage $\Delta q_{x_B} = |q_{x} -q_{x_B}|/q_x \times 100\% $, $\Delta q_{y_B} = |q_{y} -q_{y_B}|/q_y \times 100\% $, $\Delta q_{x_p} = |q_{x} -q_{x_p}|/q_x \times 100\% $ and $\Delta q_{y_p} = |q_{y} -q_{x_p}|/q_y \times 100\%$. In Figure~\ref{figure:ex2_delta_3}, we can see that the error of BULLET increases as time increases.

\subsection{Scenario 2: Cylindrical object rolling and rotating on the plane}
In this example, we consider a cylinder that rolls and rotates on the plane with 3D line contact (see Figure~\ref{figure:ex4_picture} ). The length of the cylinder is $l = 5$ m, the radius of the cylinder is $1$ m and its mass is $ m = 10$ kg. The coefficient of friction between the cylinder and the surface is $\mu = 0.3$. 

 Let the cylinder start at the configuration ${\bf q}=[0,0,1$m$,0,0,0,0]^T$ with initial generalized velocity $\bm{\nu} = [0, -1.4$ m/s$,0,0,0,0.2]^T$. In the simulation, whenever the angular velocity $^sw_z$ becomes $0$, we provide an impulse $p_{z\tau} = 3$ Nms on the cylinder. In this scenario, we set the simulation time to be 10s. 

\begin{figure*}%
\centering
\begin{subfigure}{0.65\columnwidth}
\includegraphics[width=\columnwidth]{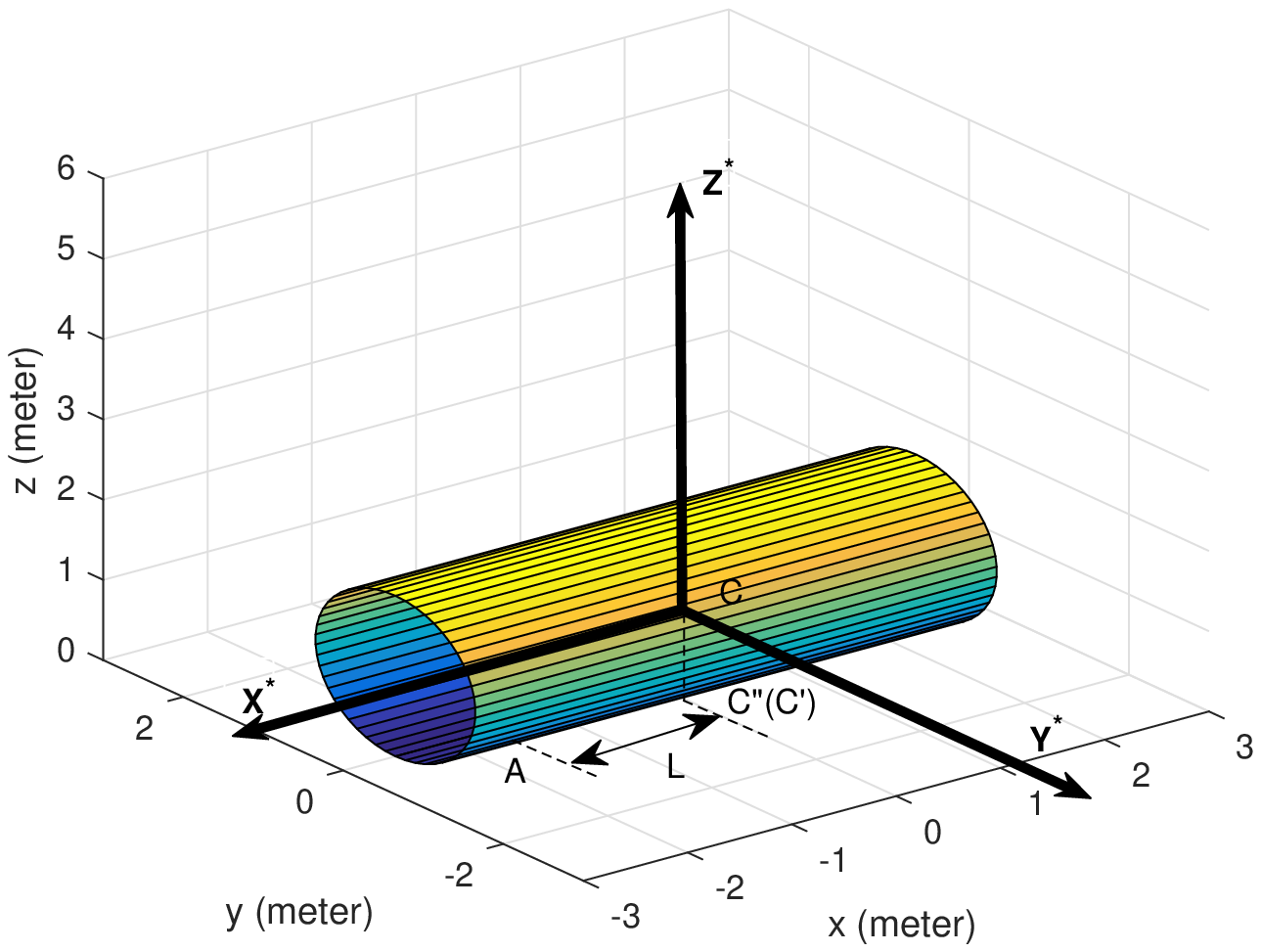}%
\caption{Cylinder on the plane. }
\label{figure:ex4_picture} 
\end{subfigure}\hfill%
\begin{subfigure}{0.65\columnwidth}
\includegraphics[width=\columnwidth]{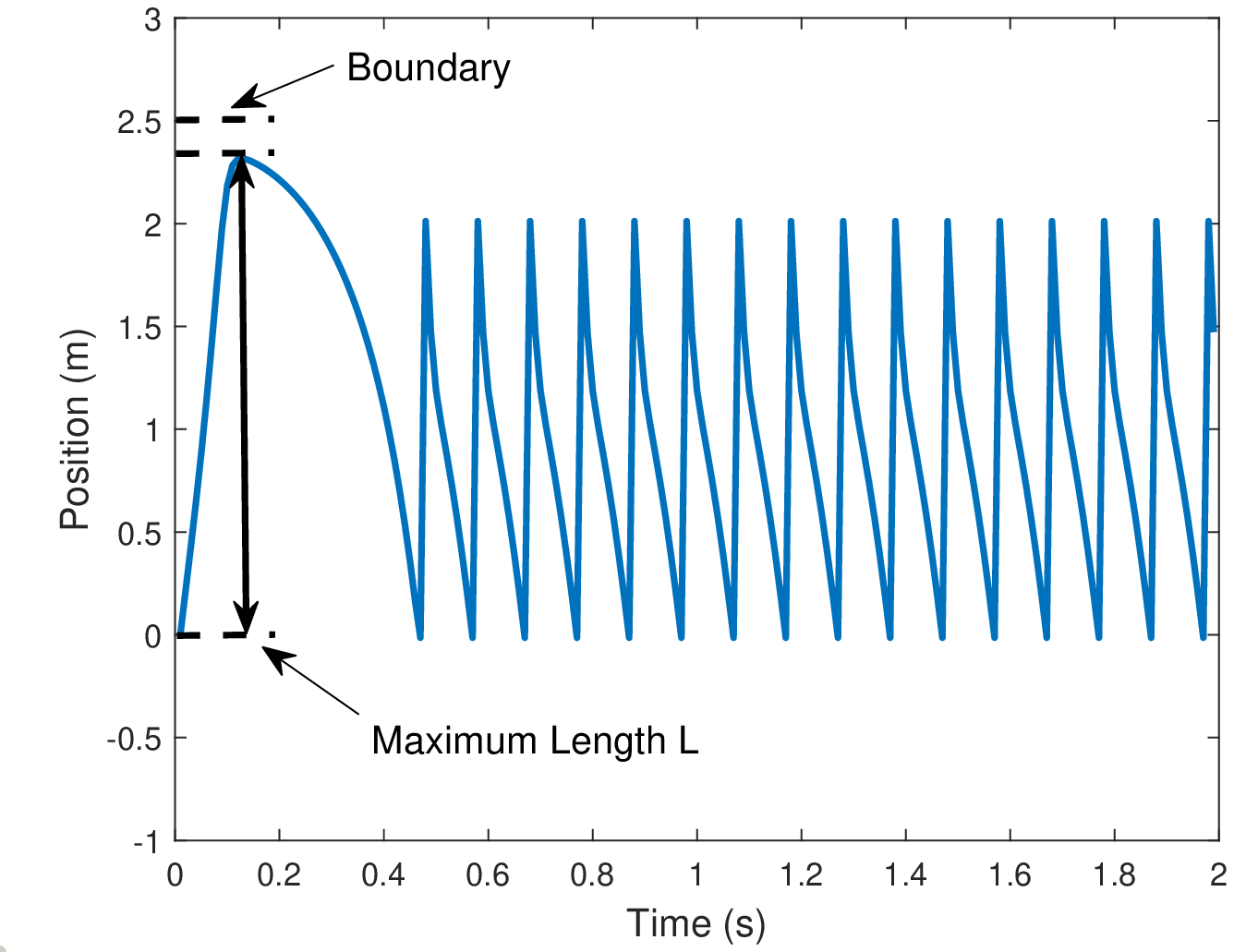}%
\caption{Distance of equivalent contact point along the contact line.}
\label{figure:ex4_al} 
\end{subfigure}\hfill%
\begin{subfigure}{0.65\columnwidth}
\includegraphics[width=\columnwidth]{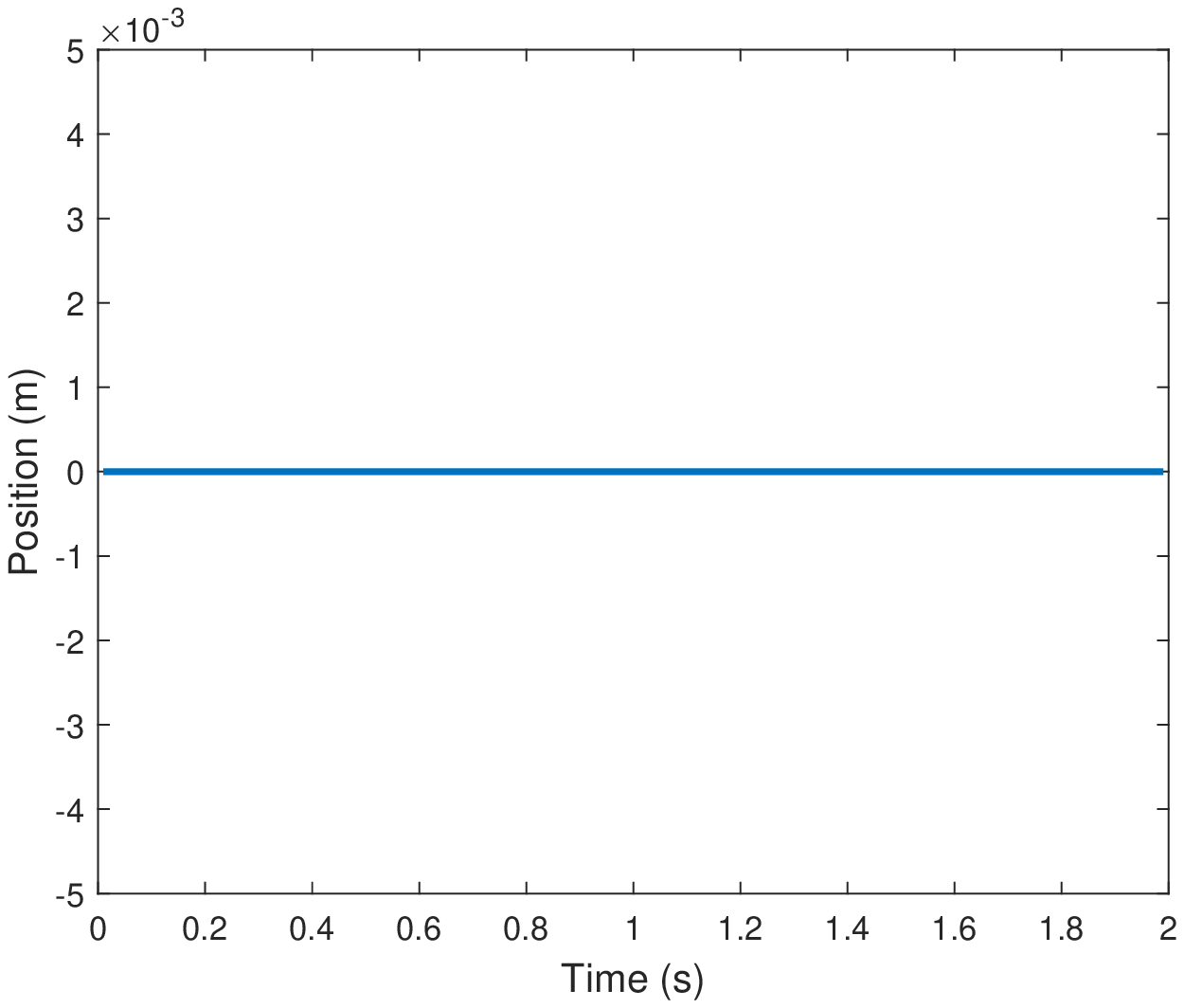}%
\caption{Distance $D$ is (numerically) zero during the motion showing that ECP is on contact line.}
\label{figure:ex4_error} 
\end{subfigure}%
\caption{Cylinder sliding and rotating on the plane.}
\label{Example2}
\end{figure*}

\begin{figure*}%
\centering
\begin{subfigure}{0.65\columnwidth}
\includegraphics[width=\columnwidth]{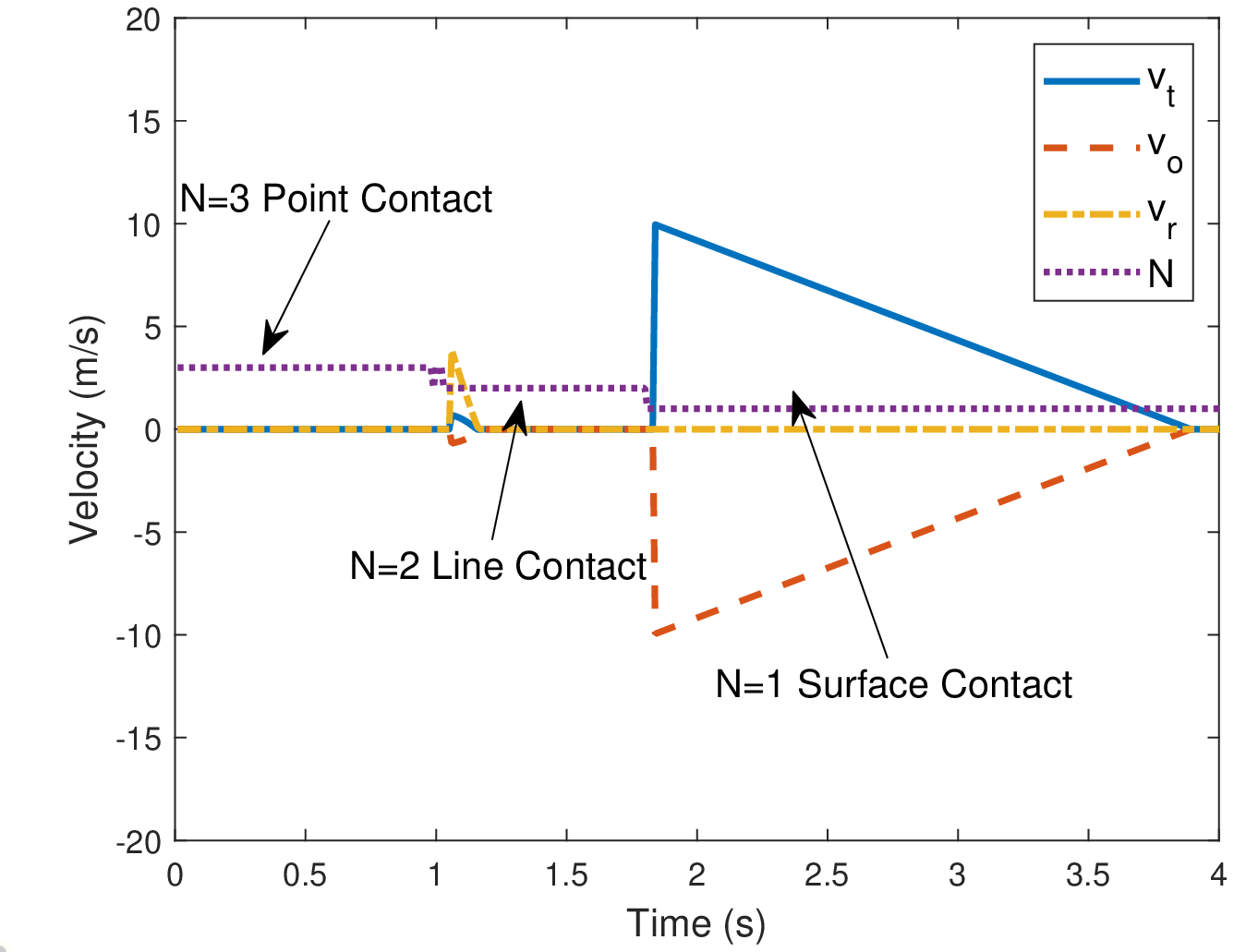}%
\caption{The sliding velocity of the ECP. }
\label{figure:ex3_sliding_velocity} 
\end{subfigure}\hfill%
\begin{subfigure}{0.65\columnwidth}
\includegraphics[width=\columnwidth]{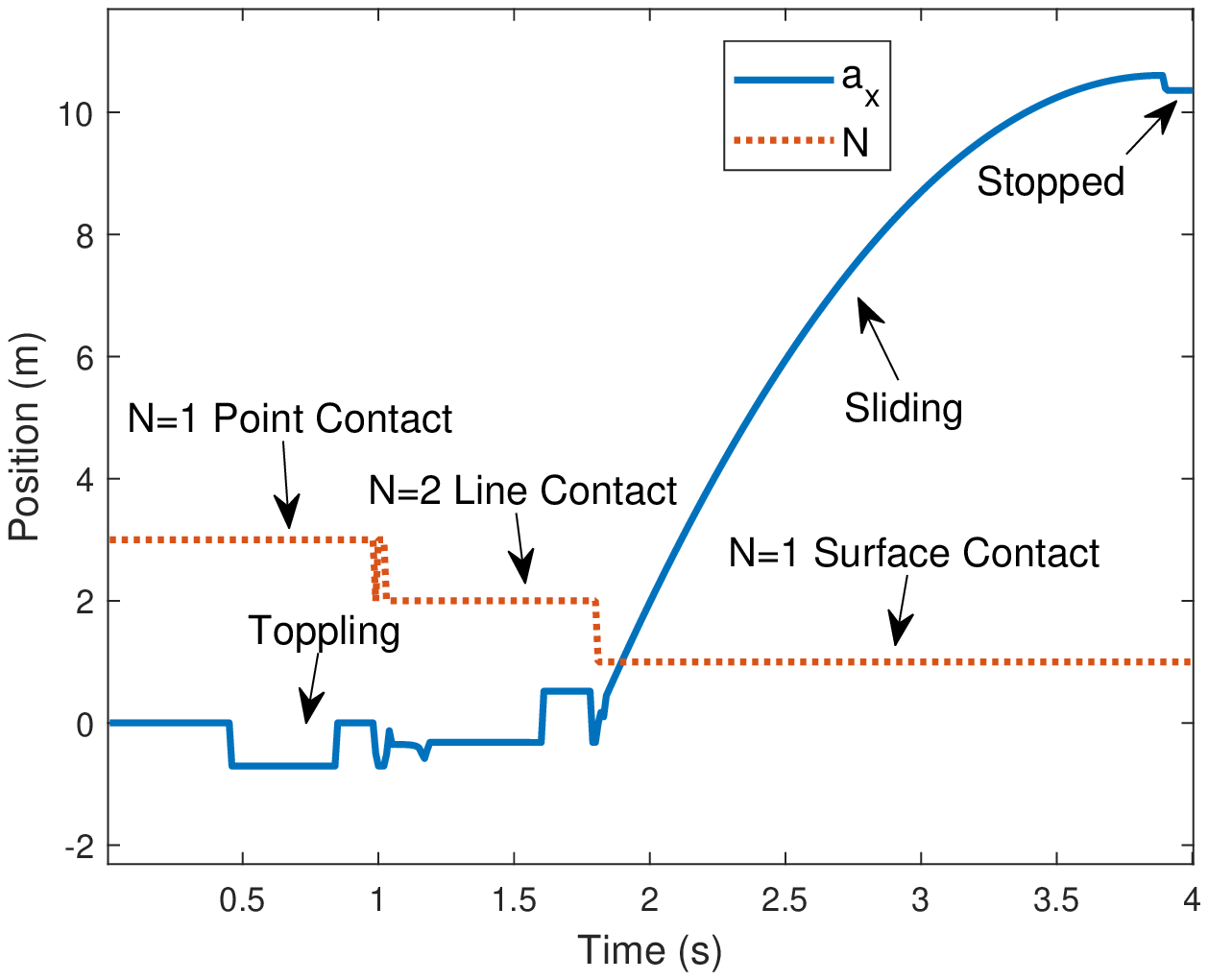}%
\caption{The x coordinate of the ECP $a_x$.  }
\label{figure:ex3_closest_point_x} 
\end{subfigure}\hfill%
\begin{subfigure}{0.65\columnwidth}
\includegraphics[width=\columnwidth]{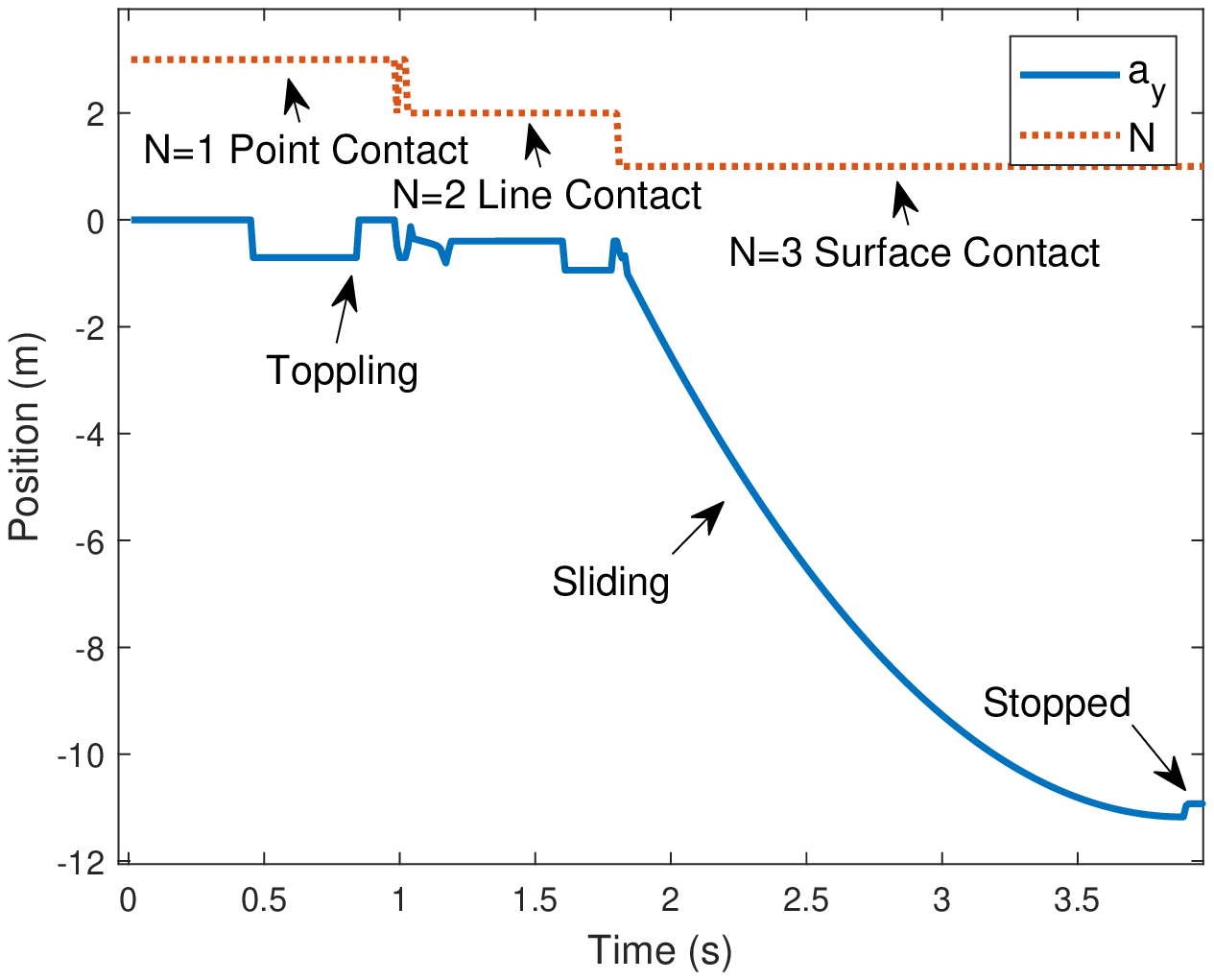}%
\caption{The y coordinate of the ECP $a_y$.}
\label{figure:ex3_closest_point_y} 
\end{subfigure}%
\caption{The motion of cube on the plane with transition among point, line, and surface contact}
\label{Example3}
\end{figure*}

When cylinder rolls and rotates on the plane with line contact, we show that the computed ECP stays on the contact line. We first project relative vector of ECP $\bf{r}$ onto the horizontal plane. Let $L$ be the relative location of ECP along the contact line:
$$L =\cos(^s\theta_z)(a_x-q_x)+\sin(^s\theta_z)(a_y-q_y),$$ 
where $^s\theta_z$ is the rotation angle about $\textbf{Z}^*$ axis. $D$ is the relative location of ECP which is orthogonal to  the contact line:
$$  D= \sin(^s\theta_z)(a_x-q_x)-\cos(^s\theta_z)(a_y-q_y)$$

Thus, ECP lies on the contact line if $ -l/2 \le L \le l/2 $ and $D = 0$. In Figure~\ref{figure:ex4_al} we show the variation of  $L$ with time. After around $t = 0.5$s, the trend changes periodically in the same pattern. Thus, we truncate the plot at $t = 2$s. From the plot, we note that $L$ is always smaller than $l/2$. When the angular velocity $^sw_z$ decreases to zero due to the friction force, $L$ also decreases to zero. When we provide impulse $p_{z\tau} = 3$Nms on the cylinder, $^sw_z$ increases, and $L$ also increases. Figure~\ref{figure:ex4_error} shows the trend of the distance $D$, which is truncated at $t = 2$s. We can see that the distance $D$ is always $0$. Note that, after $t = 2$s, the distance $D$ keeps to be $0$. Thus, we can conclude that the ECP is on the line of contact when the cylinder is rolling and rotating on the plane.

\subsection{Scenario 3: The cube toppling and sliding on the plane}
In this example we show that our method can automatically transition among point, line, and surface contact and solve for the ECP. We consider a unit cube starting with point contact with a plane. Let the coefficient of friction be $\mu = 0.2$ and cube's initial configuration be: 
${\bf q} = [0, 0, \frac{\sqrt{3}}{2}$ m$,\cos\left(\frac{\theta}{2}\right), \frac{1}{\sqrt{2}}\sin\left(\frac{\theta}{2}\right), -\frac{1}{\sqrt{2}}\sin\left(\frac{\theta}{2}\right), 0]^T$. $\theta = {\rm arctan}(\sqrt{2})$. The initial generalized velocity is $\bm{\nu} = [-\sqrt{6}/4 $  m$, -\sqrt{6}/4$ m$, 0, \sqrt{2}/2$  rad/s$, -\sqrt{2}/2$ rad/s$, 0]^T$. In this scenario, we set the simulation time to be 4s.

Figure~\ref{figure:ex3_sliding_velocity} shows the sliding velocities $v_t, v_o, v_r$ and $N$, the number of nonzero Lagrange multipliers denoting the number of facets of the cube that are in contact. In Equation~\eqref{equation:re_contact_multiple_3}, when one contact surface touches plane, its associated $l_i > 0$. When a cube has point contact with the plane, it has three facets which contact with plane, so $N = 3$. Similarly we have $N =2$ for line contact and $N =1$ for surface contact. 
In the beginning, the cube has point contact with plane. During motion, when the point contact becomes line contact ($t = 1$s), we provide an applied impulse, ${\bf P}_{app} = [\sqrt{2}/2$Ns $,-\sqrt{2}/2$Ns $,0,0.5$Nms $,0.5$Nms $,0]$ on the cube. Finally, when the cube has surface contact with plane ($t = 1.8$s), we provide an applied impulse ${\bf P}_{app} = [10$ Ns $,-10$Ns$,0,0,0,0]$ on the cube to let it slide on the plane. 
In Figure~\ref{figure:ex3_closest_point_x} and~\ref{figure:ex3_closest_point_y}, we plot $a_x$ , $a_y$ and $N$. During point contact and line contact, ECP suddenly changes several times, which means the cube oscillates between nearby points or lines but still keeps point or line contact.

\begin{figure*}[!htp]%
\begin{subfigure}{0.25\textwidth}
\includegraphics[width=\textwidth]{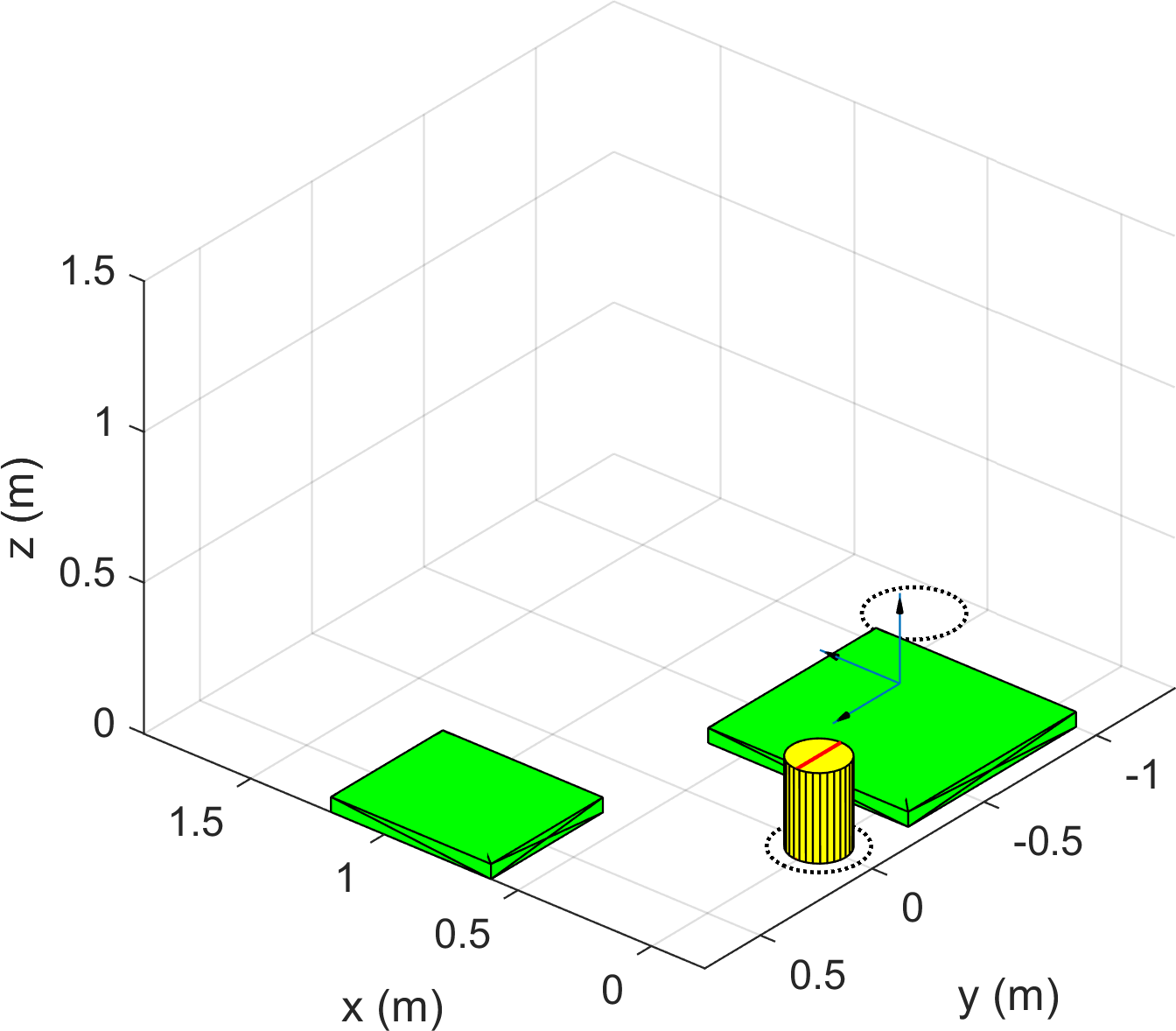}%
\caption{t = 0.01s ($\text{T}_1$). }
\label{figure:mani_1} 
\end{subfigure}\hfill%
\begin{subfigure}{0.25\textwidth}
\includegraphics[width=\textwidth]{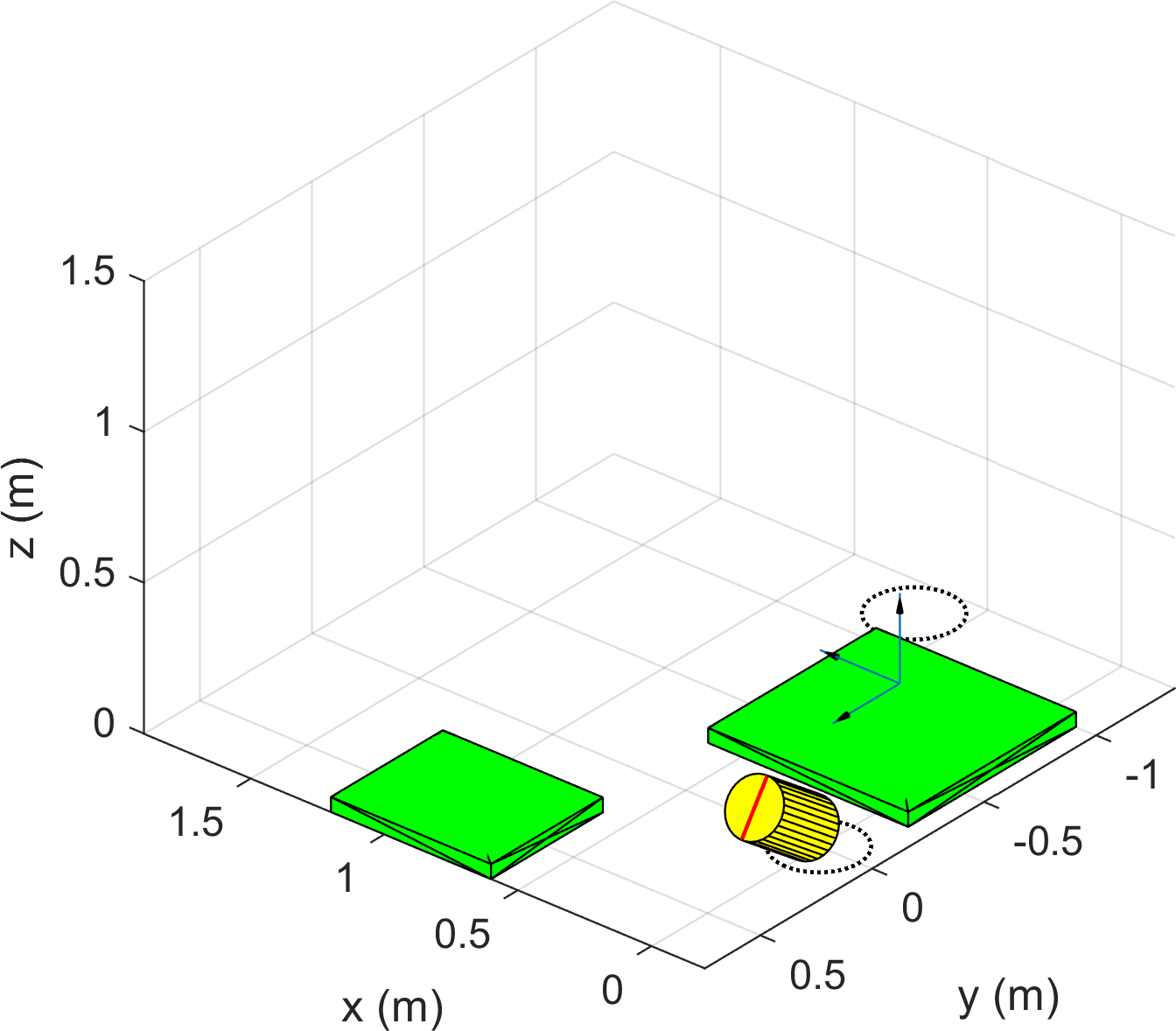}%
\caption{t = 0.51s ($\text{T}_1$). }
\label{figure:mani_2} 
\end{subfigure}\hfill%
\begin{subfigure}{0.25\textwidth}
\includegraphics[width=\textwidth]{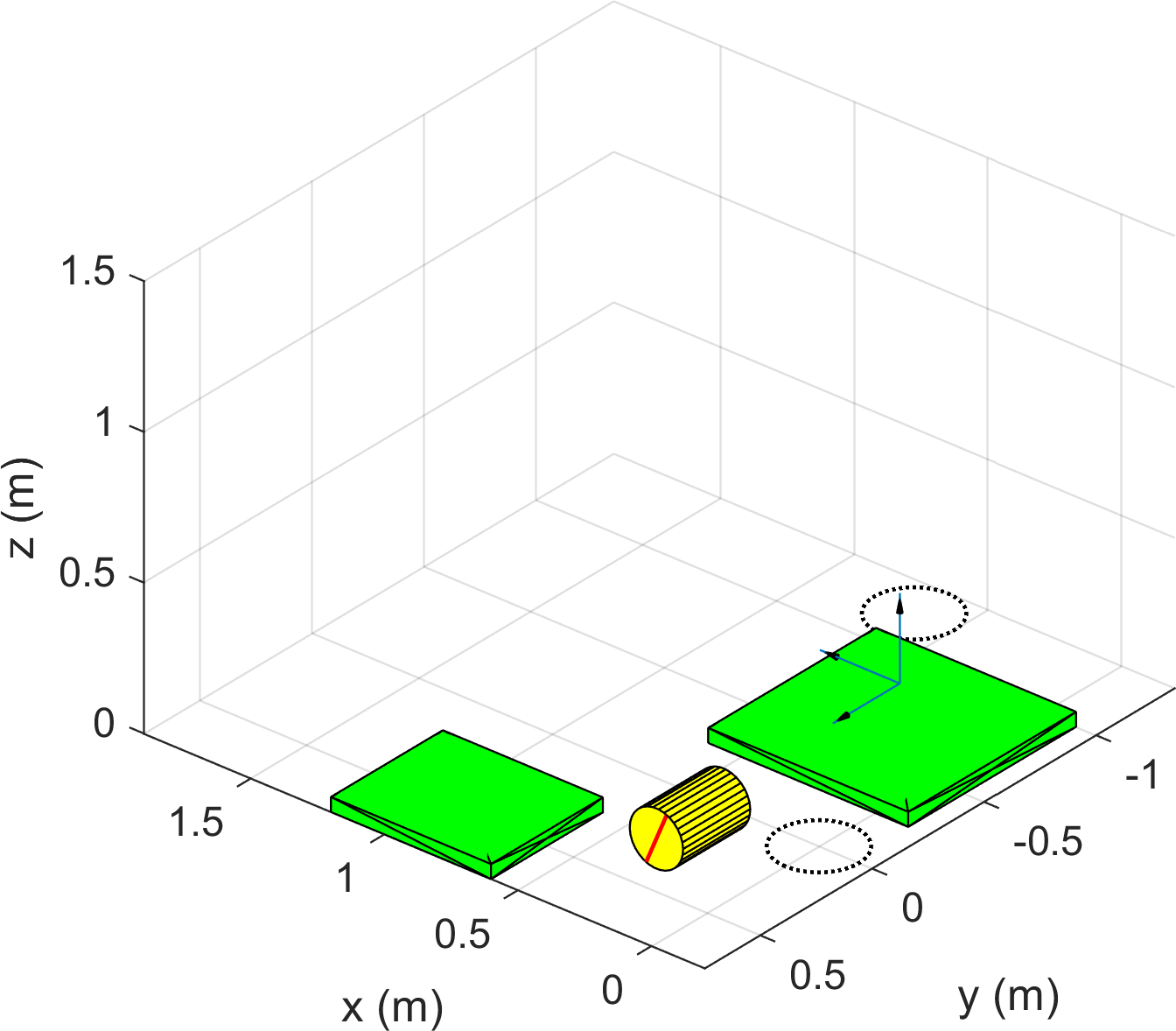}%
\caption{t = 1.31s ($\text{T}_2$).}
\label{figure:mani_3} 
\end{subfigure}\hfill%
\begin{subfigure}{0.25\textwidth}
\includegraphics[width=\textwidth]{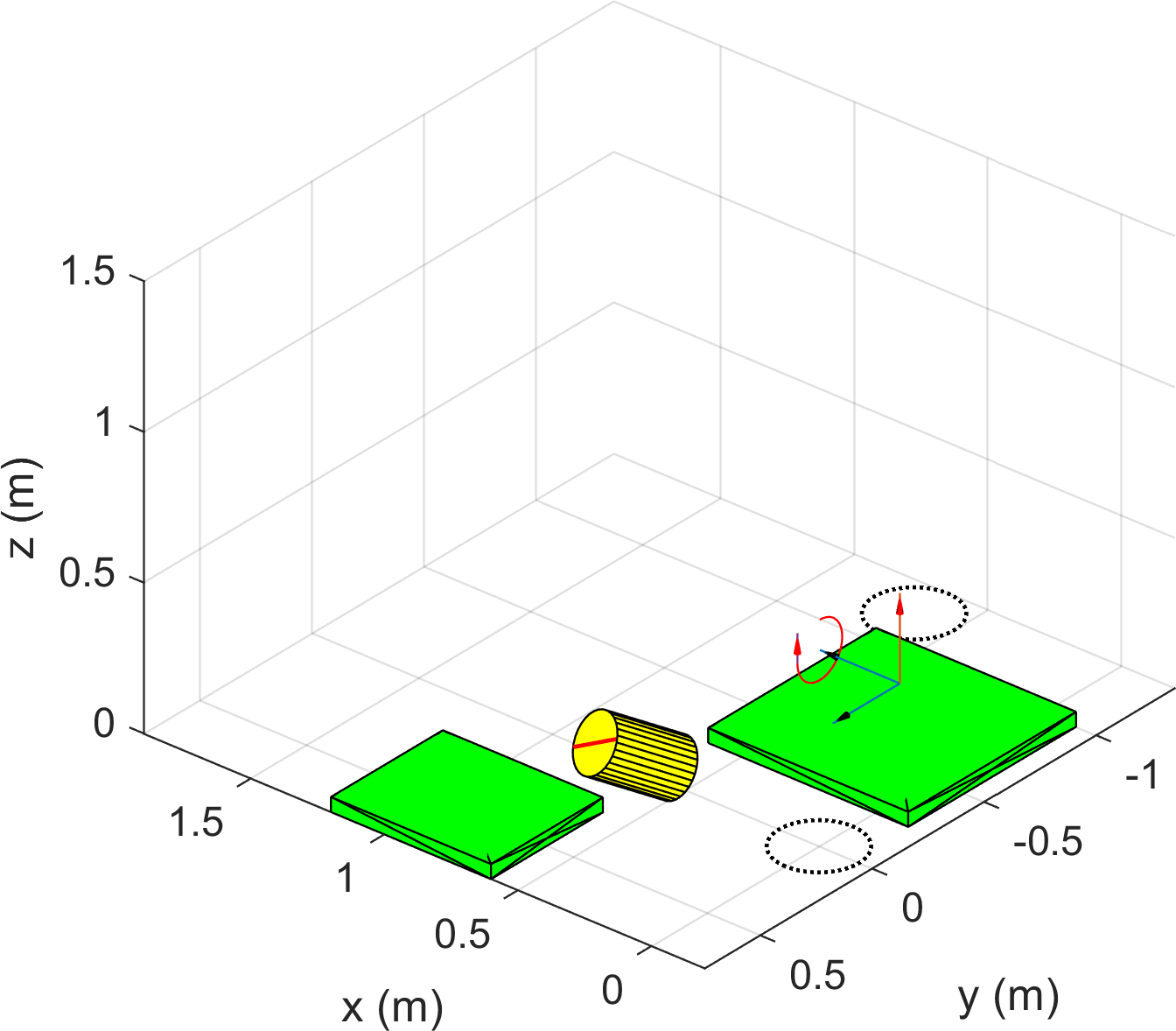}%
\caption{t = 1.81s ($\text{T}_3$).}
\label{figure:mani_4} 
\end{subfigure}\hfill%
\begin{subfigure}{0.25\textwidth}
\includegraphics[width=\textwidth]{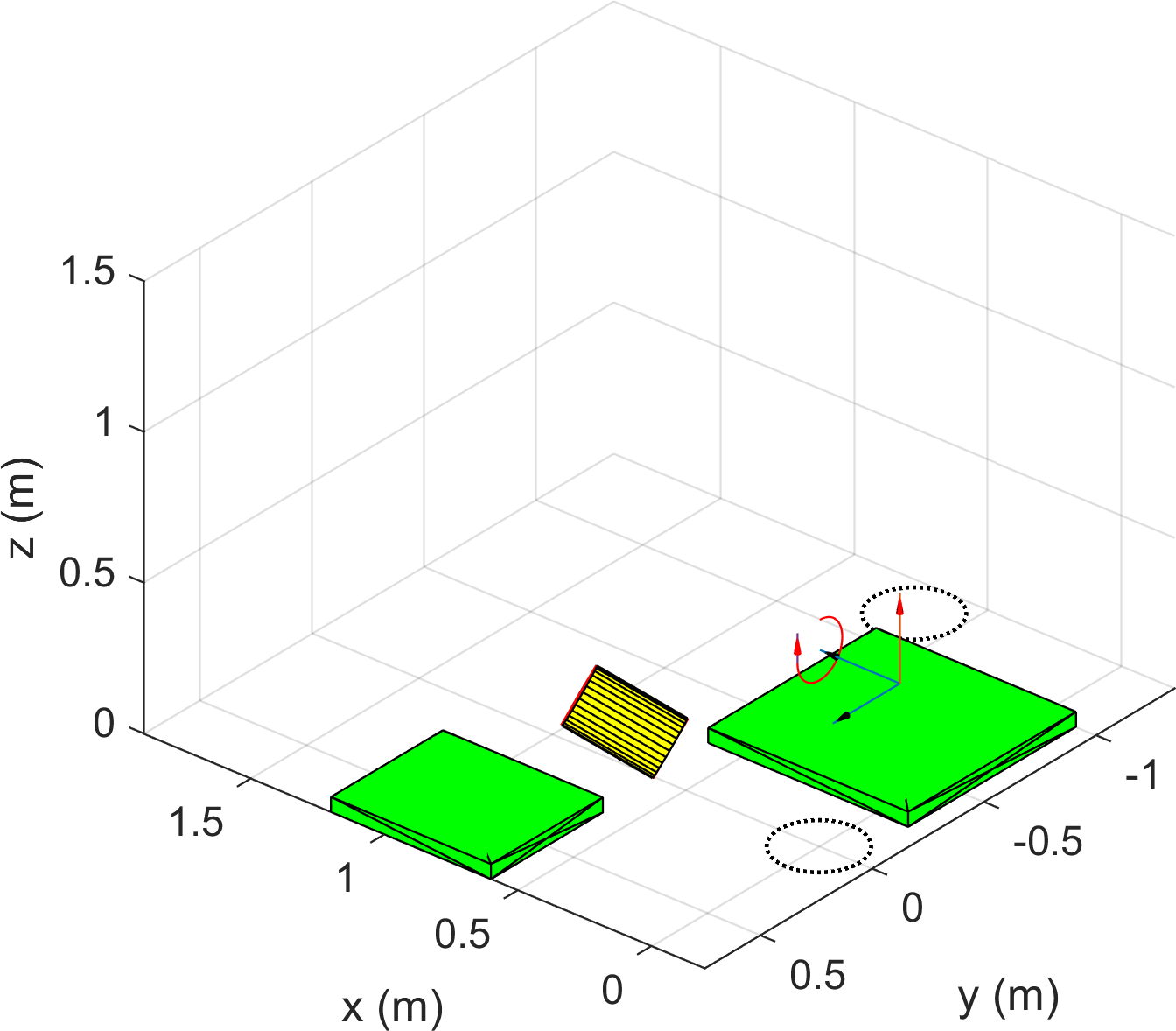}%
\caption{t = 2.01s ($\text{T}_3$).}
\label{figure:mani_5} 
\end{subfigure}\hfill%
\begin{subfigure}{0.25\textwidth}
\includegraphics[width=\textwidth]{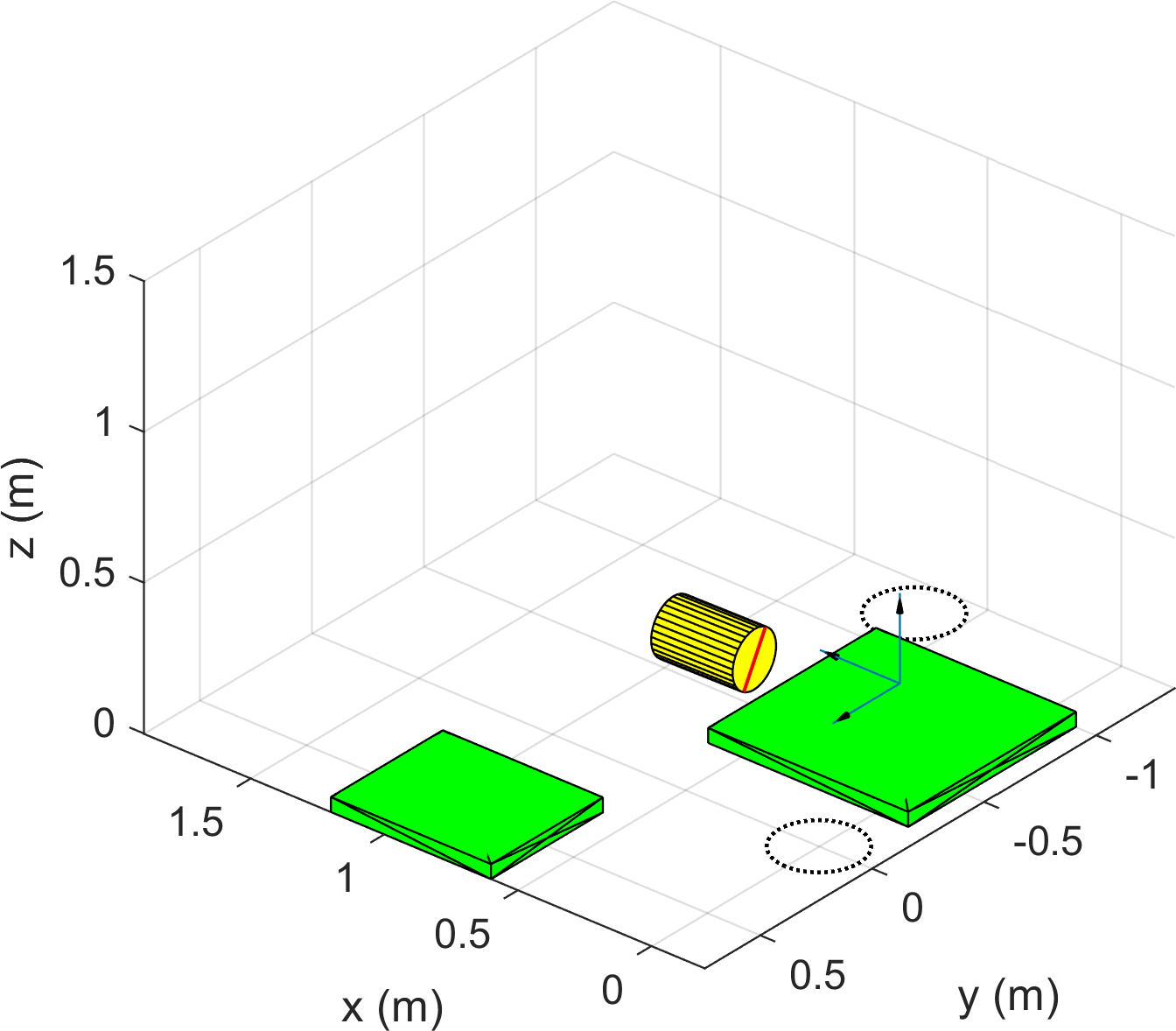}%
\caption{t = 2.51s ($\text{T}_4$).}
\label{figure:mani_6} 
\end{subfigure}\hfill%
\begin{subfigure}{0.25\textwidth}
\includegraphics[width=\textwidth]{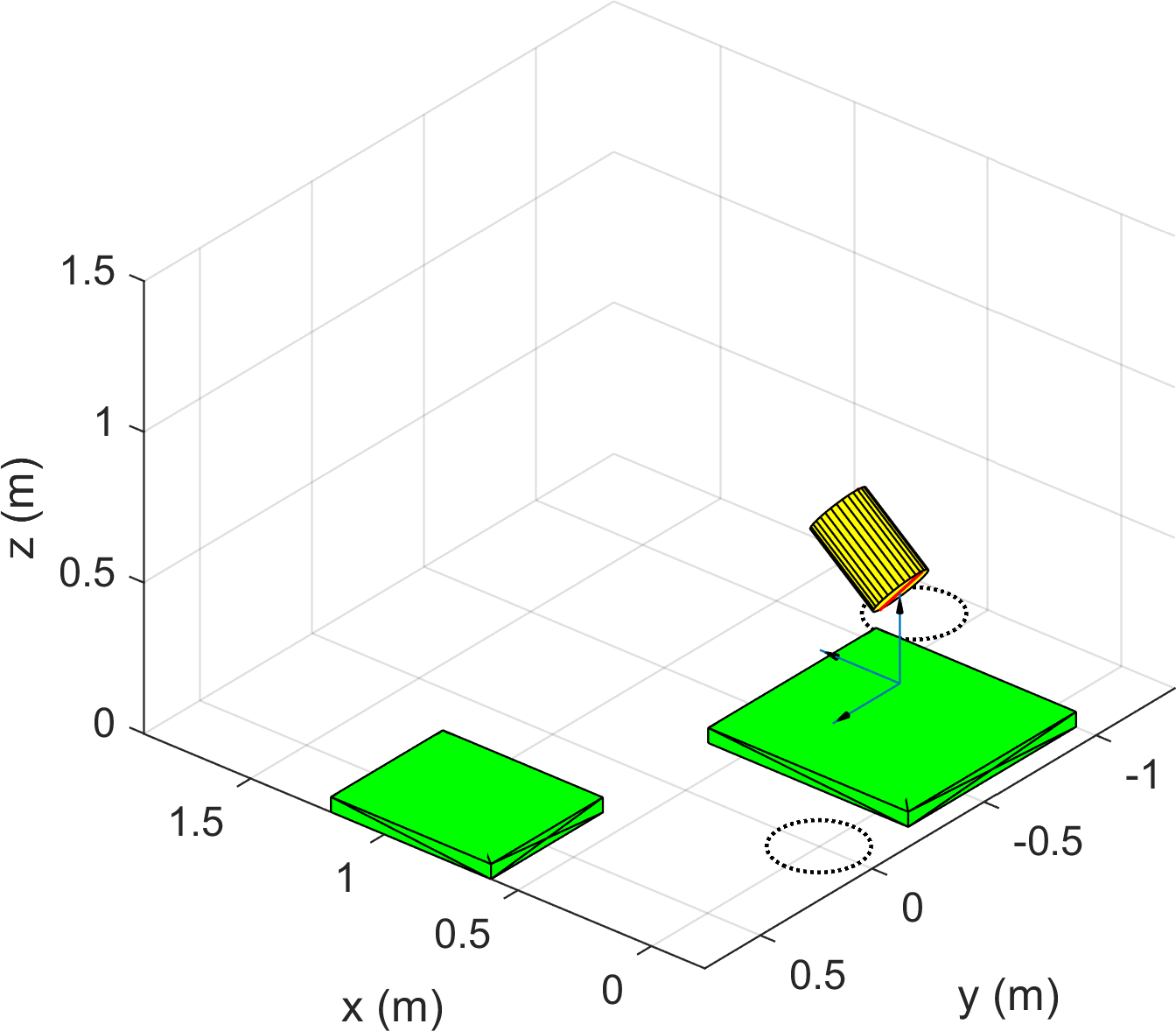}%
\caption{t = 3.51s ($\text{T}_5$).}
\label{figure:mani_7} 
\end{subfigure}\hfill%
\begin{subfigure}{0.25\textwidth}
\includegraphics[width=\textwidth]{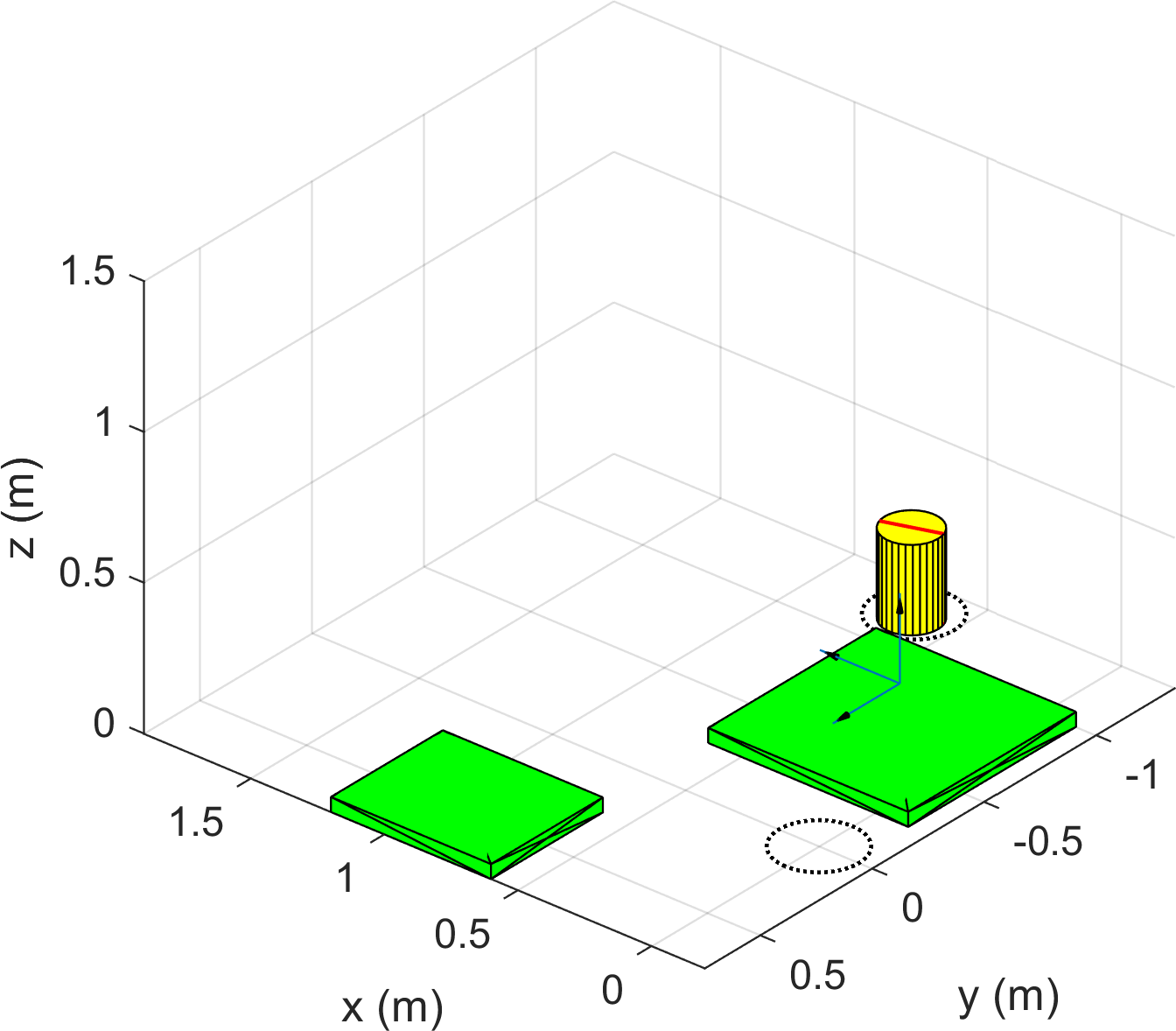}%
\caption{t = 5s ($\text{T}_6$).}
\label{figure:mani_8} 
\end{subfigure}
\caption{Simulations for the manipulation task of the cylindrical object showing transitions among point, line and surface contact. Note that $\text{T}_1$ to $\text{T}_6$ correspond to time periods in order during the manipulation. }
\label{Example_Manipulation_1}
\end{figure*}

\begin{figure*}[!htp]%
\begin{subfigure}{1\columnwidth}
\includegraphics[width=\columnwidth]{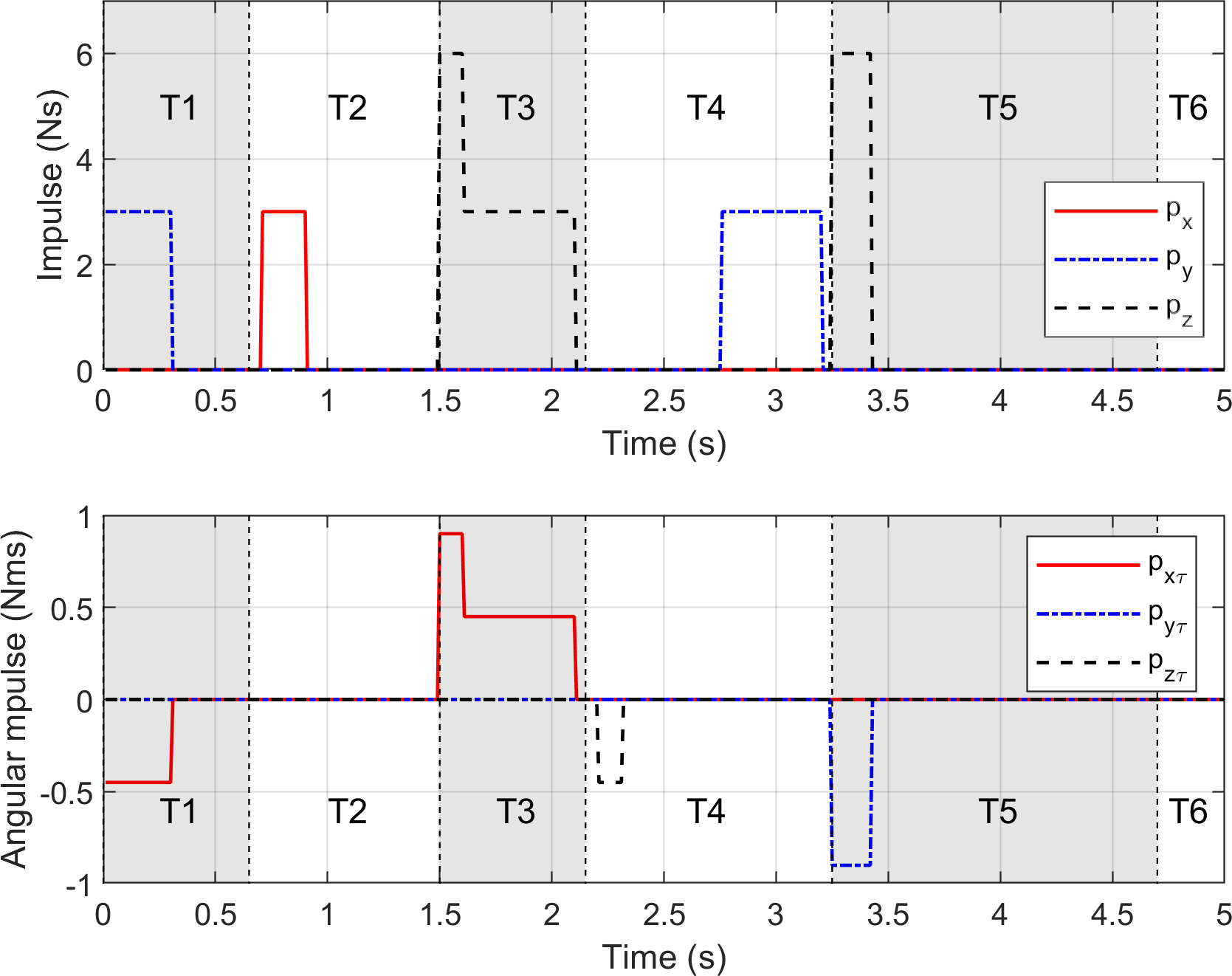}%
\caption{}
\label{figure:manpulation_impulse} 
\end{subfigure}\hfill%
\begin{subfigure}{1\columnwidth}
\includegraphics[width=\columnwidth]{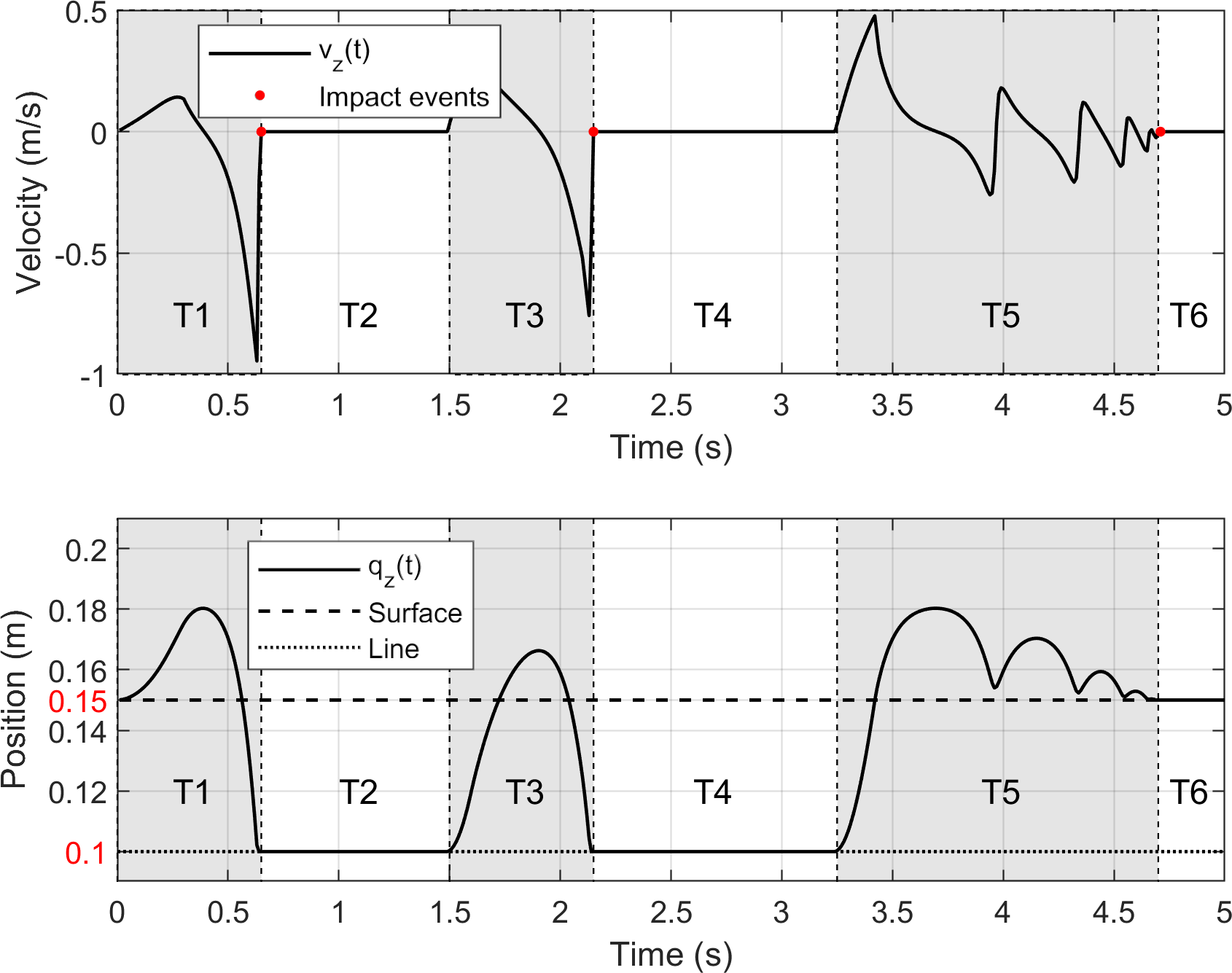}%
\caption{}
\label{figure:manpulation_trajectory} 
\end{subfigure}
\caption{(a) The time series of the applied impulses ${\bf p}_{app}(t) = [p_x(t),p_y(t),p_z(t),p_{x\tau}(t),p_{y\tau}(t),p_{z\tau}(t)]^T$ on the object. (b) The trajectory of velocity $v_z(t)$ and position $q_z(t)$ of the object's CM. In the plots, we divide the time trajectories into the time periods from $\text{T}_1$ to $\text{T}_6$. In the top of (b), when $v_z(t)$ drops to zero (dots in red), the object has inelastic contact with the plane. In the bottom of (b), the object has line contact when $q_z(t) = 0.1$m (dotted line). And it has surface contact when $q_z(t) = 0.15$m (dashed line). And the object has point contact when it transits between line and surface contacts. }
\label{Example_Manipulation_2} 
\end{figure*}
\subsection{Scenario 4: The manipulation task of cylindrical object}
This scenario illustrates the simulation for a manipulation task based on our method. This example is chosen to show that the actions of rolling and pivoting can be useful in the manipulation task. As shown in Figure~\ref{figure:mani_1}, initially a cylindrical object stands within a dashed circle on the plane. And there exists two cuboid shape obstacles which may block the motion of the object. Our goal is to move it to the goal configuration (shown in Figure~\ref{figure:mani_8}). The cylinder is made of steel with mass $m = 75$kg. The object's length is $0.3$m and the radius is $0.1$m. The plane is made of rubber, and the coefficient of friction between it and the object is $\mu = 0.8$. So the friction resistance is high if we directly push the object on the rough surface. And grasping is also hard since the object is heavy. Thus, the manipulations like pivoting and rolling is preferred.

Figure~\ref{Example_Manipulation_1} demonstrates the snapshots for the manipulation task.
Shown in Figure~\ref{figure:mani_1} and~\ref{figure:mani_2}, we first make the object fall down on the plane ($\text{T}_1$: from t = 0.01s to t = 0.65s). In Figure~\ref{figure:mani_3}, the object rolls forward with line contact ($\text{T}_2$: from t = 0.66s to t = 1.5s).  However, one of the obstacles (the cuboid shape in green) blocks the motion. As shown in Figures~\ref{figure:mani_4} and~\ref{figure:mani_5}, we make the object pivot and rotate about the contact point and it successfully passes through the obstacle ($\text{T}_3$: from t = 1.51s to t = 2.15s). In Figure~\ref{figure:mani_6}, we make the object rolls forward ($\text{T}_4$: from t = 2.16s to t = 3.25s). In Figure~\ref{figure:mani_7}, we lift the object up ($\text{T}_5$: from t = 3.26s to t = 4.7s). In Figure~\ref{figure:mani_8}, eventually the object reaches the goal configuration ($\text{T}_6$: from t = 4.71s to t = 5s).

The time series of applied impulses ${\bf p}_{app}(t)$ on the object is shown in Figure~\ref{figure:manpulation_impulse}. In Figure~\ref{figure:manpulation_trajectory}, we compare two plots of velocity $v_z(t)$ and position $q_z(t)$ of the object's CM. We show that the timings of jumps in $v_z(t)$ (identified by the red dots) correspond to the timings where $q_z(t)$ drops to constant value. Thus, we can conclude that the impact events happen when the contact mode changes. In our model, 
the collision is inelastic. Thus, the velocity component $v_z(t)$ goes to zero when impact happens. Note that both $v_z(t)$ and $q_z(t)$ vibrate multiple times during $\text{T}_5$ period. And it is reasonable since the object wobbles on the plane after it has been lifted up during $\text{T}_5$.

%% file: 7_conclusion.tex
\section{CONCLUSIONS}
\label{sec:conc}
In this paper, we present a geometrically implicit time-stepping method for solving dynamic simulation problems with intermittent contact where the contact may be line or surface contact. In our geometrically implicit method we formulate a nonlinear mixed complementarity problem for each time step that allows us to solve the collision detection and numerical integration of the equations of motion simultaneously, instead of decoupling them (as is traditionally done). Decoupling of the collision detection from equations of motion makes the collision detection problem ill-posed for line or surface contact because there are infinitely many possibilities for contact point choice. Combining the collision detection with numerical integration allows us to solve for an {\em equivalent contact point} (ECP) on the contact patch as well as the contact wrenches simultaneously and makes the problem well-posed. We present numerical simulation results for some manipulation examples that demonstrate that our method can automatically handle and simulate through changes in point, line, and surface contact modes.
%In special cases of pure translation and pure rotation, we can solve for the ECP and contact wrenches in closed form as well as present conditions under which the contact mode (i.e., line or surface contact) will be maintained at the end of the time step. In future work, we want to explore the use of these conditions for manipulation planning and control.
% * <jiayin.xie@stonybrook.edu> 2017-09-22T20:11:01.342Z:
% 
% In this paper, the special cases are not mentioned.  
% 
% ^.

%Our solution assumes that the objects are modeled as convex sets and the contact patch is a convex set. In the future, we will also like to extend this work to non-convex contact patches. In~\cite{XieC18}, we have extended the current methodology to nonconvex objects and nonconvex contact patches, where the object and contact patch can be modeled as a union of convex sets.